\documentclass[twoside,11pt]{article}

\usepackage{blindtext}

\usepackage[abbrvbib, preprint]{jmlr2e}

\usepackage{hyperref}
\usepackage{booktabs}
\usepackage{algorithm}
\usepackage{algpseudocode}
\usepackage{amsmath}
\usepackage{subfig}
\usepackage{xcolor}
\usepackage{url}
\usepackage{mathabx}
\usepackage{multirow}

\newtheorem{defn}{Definition}
\newtheorem{thrm}{Theorem}
\newtheorem{lem}{Lemma}

\newcommand{\minisection}[1]{\vspace{2mm}\noindent{\textbf{#1}}}
\DeclareMathOperator*{\argmax}{argmax}

\usepackage{lastpage}
\jmlrheading{25}{2024}{1-\pageref{LastPage}}{6/24; Revised 7/24}{8/24}{21-0000}{Duc C. Hoang, Behzad Ousat, Amin Kharraz and Cuong V. Nguyen}

\ShortHeadings{EnSolver: Uncertainty-Aware CAPTCHA Solvers with Theoretical Guarantees}{Hoang, Ousat, Kharraz and Nguyen}
\firstpageno{1}

\begin{document}

\title{EnSolver: Uncertainty-Aware Ensemble CAPTCHA Solvers with Theoretical Guarantees}

\author{\name Duc C. Hoang \email hcduc@comp.nus.edu.sg \\
       \addr Department of Computer Science \\
       National University of Singapore, Singapore
       \AND
       \name Behzad Ousat \email bousa001@fiu.edu \\
       \addr Knight Foundation School of Computing and Information Sciences \\
       Florida International University, USA
       \AND
       \name Amin Kharraz \email ak@cs.fiu.edu \\
       \addr Knight Foundation School of Computing and Information Sciences \\
       Florida International University, USA
       \AND
       \name Cuong V. Nguyen \email viet.c.nguyen@durham.ac.uk \\
       \addr Department of Mathematical Sciences \\ 
       Durham University, UK}

\editor{My editor}

\maketitle

\begin{abstract}
The popularity of text-based CAPTCHA as a security mechanism to protect websites from automated bots has prompted researches in CAPTCHA solvers, with the aim of understanding its failure cases and subsequently making CAPTCHAs more secure. Recently proposed solvers, built on advances in deep learning, are able to crack even the very challenging CAPTCHAs with high accuracy. However, these solvers often perform poorly on out-of-distribution samples that contain visual features different from those in the training set. Furthermore, they lack the ability to detect and avoid such samples, making them susceptible to being locked out by defense systems after a certain number of failed attempts. In this paper, we propose EnSolver, a family of CAPTCHA solvers that use deep ensemble uncertainty to detect and skip out-of-distribution CAPTCHAs, making it harder to be detected. We prove novel theoretical bounds on the effectiveness of our solvers and demonstrate their use with state-of-the-art CAPTCHA solvers. Our experiments show that the proposed approaches perform well when cracking CAPTCHA datasets that contain both in-distribution and out-of-distribution samples.\footnote{The source code for this paper is available at: \url{https://github.com/HoangCongDuc/ensolver.git}.}
\end{abstract}


\section{Introduction}
\label{sec:sample1}

Automated web bots are getting increasingly more sophisticated in imitating human behaviors and evading detection. In most important and consequential situations, these adversarial operations can result in credential stuffing, account hijacking and data breaches, vulnerability scanning and exploitation. 
CAPTCHA (Completely Automated Public Turing test to tell Computers and Humans Apart) \citep{von2003captcha} has been a common security mechanism to defend against malicious automatic bot activities. That is, the website generates a CAPTCHA challenge and requests the remote agent to solve the CAPTCHA. The core insight is that solving a CAPTCHA is straightforward for real users but difficult for automated bots. Among different types of CAPTCHAs, text-based CAPTCHAs~\citep{gao2016robustness} present users with a noisy image of a short text string and ask them to enter the correct string. They are among the most popular types of CAPTCHAs due to simple implementation and user-friendliness~\citep{deng3e}.

In this adversarial landscape, as text-based CAPTCHAs became more popular, many automatic solving techniques have also been developed to evade detection. These techniques are interesting from the security viewpoint as they help researchers understand the weaknesses of these CAPTCHAs and subsequently make them more secure~\citep{tang2018research}. With recent advances in deep learning and computer vision, state-of-the-art text-based CAPTCHA solvers often employ an end-to-end approach, where a sophisticated deep learning model will predict the output text string directly from the raw pixels of the input image~\citep{deng3e, ousat2024matter}. Although these solvers can crack very challenging CAPTCHAs, they are unable to make correct predictions on out-of-distribution samples (i.e., images that are visually different from those in their training set). By exploiting this limitation, defense systems can detect these automatic CAPTCHA solvers and lock their access if a solver fails a certain number of attempts. Thus, from the attacker's perspective, it is desirable to equip the solvers with the ability to recognize and skip these out-of-distribution samples and request a new CAPTCHA to increase their success rate.

In this paper, we propose EnSolver (\underline{En}semble \underline{Solver}), a simple end-to-end text-based CAPTCHA solver that is capable of detecting and skipping out-of-distribution CAPTCHAs. EnSolver uses a deep ensemble~\citep{lakshminarayanan2017simple} to quantify its predictive uncertainty, subsequently deciding whether to answer a given CAPTCHA or to skip and request a new one. More specifically, EnSolver computes a distribution over predicted text strings and uses this distribution to predict the final text string along with its uncertainty estimation. The solver then decides whether to skip the input image based on this quantity. To use the EnSolver in a real scenario, we extend it to LEnSolver (\underline{L}imited-skip \underline{EnSolver}), where we allow the solver to skip a maximum number of times before forcing it to make a prediction. Our approach is general and can be used with any type of base models in the deep ensemble.

Besides the general algorithmic framework, we also develop new theoretical properties for our EnSolver and LEnSolver approaches. In particular, we prove a lower bound on the rate of making the right decisions of the EnSolver models and a lower bound on the success rate of the LEnSolver models. Both of these bounds depend on a novel quantity, called the out-of-distribution error bound, that is defined in terms of the ensemble size, the output domain size, and the uncertainty threshold of the solver. This out-of-distribution error bound can be shown to upper bound the EnSolver's error rate on out-of-distribution data as well as the EnSolver's wrong prediction rate on in-distribution data. This quantity is usually very small in real CAPTCHA solving problems, thus helping us control the error rates of our solvers.

We apply our approaches to two state-of-the-art types of text-based CAPTCHA solvers: the 3E-Solver~\citep{deng3e} and the object detection-based solver~\citep{ousat2024matter}, both of which employ a complex deep learning architecture without any uncertainty estimation mechanism. To train our solvers, we also construct a new CAPTCHA dataset that contains bounding box labels, which are required to train object detection models. We evaluate the good decision and success rates of our solvers on datasets containing both in-distribution and out-of-distribution data, where the latter are collected from eight different public CAPTCHA datasets~\citep{deng3e}. The experiment results show that our solvers are consistently better than the baselines, and the theoretical lower bounds can give us a good idea of the actual empirical performance in practice. 

\minisection{Contributions.} Our paper makes the following contributions:
\begin{enumerate}
    \item We propose EnSolver and LEnSolver, novel uncertainty-aware CAPTCHA solvers that use deep ensemble uncertainty to avoid making wrong predictions on out-of-distribution inputs (Section~\ref{sec:EnSolver_LEnSolver}).
    \item We prove new theoretical lower bounds on the effectiveness of our solvers (Section~\ref{sec:theory}) and show through experiments that these bounds can be good indicators of our solvers' empirical performance in practice (Section~\ref{sec:experiment}). 
    \item We demonstrate the use of our approaches on two state-of-the-art CAPTCHA solvers and show empirically that our solvers can achieve high success rates compared to the two original baselines (Section~\ref{sec:experiment}).
\end{enumerate}

\begin{figure*}[tb]
    \centering
    \includegraphics[width=\textwidth]{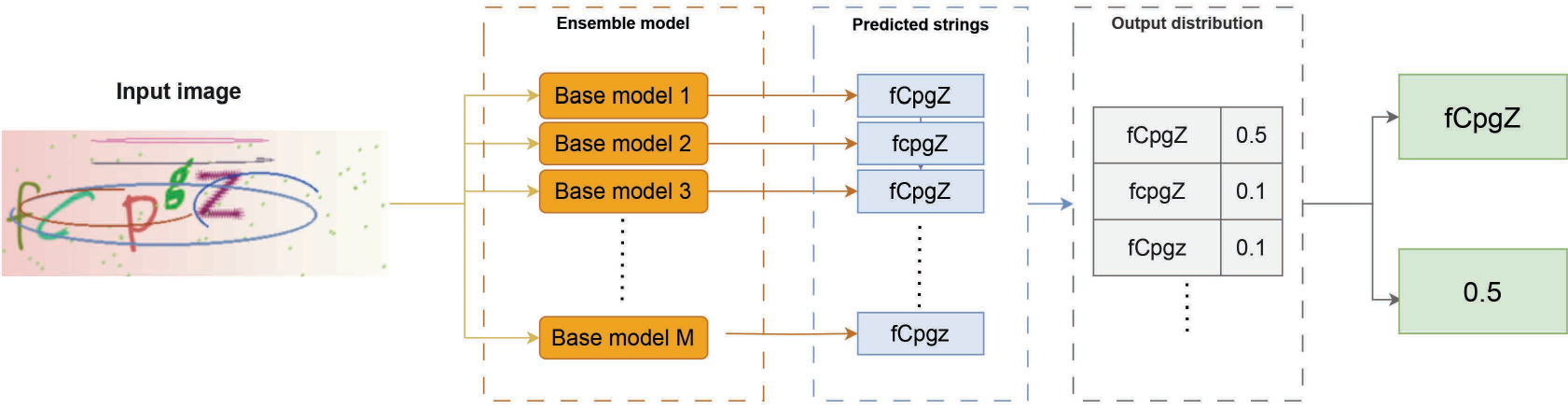}
    \caption{The main component of EnSolver that predicts an output string together with an associated uncertainty level given an input CAPTCHA image. The input image is first fed into each base model, each of which produces a string as output. The output strings form a distribution, which is used to compute the final prediction and the uncertainty level.}
    \label{fig:ensolver}
\end{figure*}

\section{Related Work}

\minisection{CAPTCHA Solvers.}
Early solvers often use a segmentation-based approach consisting of two main stages: character segmentation and character recognition~\citep{chellapilla2004using, yan2007breaking, yan2008low}. \citet{yan2007breaking} used the color difference to localize characters in simple 2-color CAPTCHAs. Their subsequent work~\citep{yan2008low} used a histogram of foreground pixels to vertically segment the characters into chunks and then find the connected components to get individual characters. Once the characters are segmented, it is easy to recognize them individually using a neural network~\citep{chellapilla2004using, kopp2017breaking}. Since segmentation is a crucial stage in many solvers, a number of anti-segmentation features (e.g., character overlapping, hollow scheme, noise arcs, and complicated background) have been employed to make CAPTCHAs more resistant against such solvers~\citep{tang2018research}. To bypass these anti-segmentation features, various preprocessing techniques were then developed. For example, \citet{gao2016robustness} proposed using image binarization, noise arcs removal, and image rectification before segmentation, while \citet{ye2018yet} used an image-to-image translation model~\citep{isola2017image} to turn a noisy input CAPTCHA into an easier image for the downstream solver. \citet{tian2020generic} used three generator networks to decompose an input CAPTCHA into a background and a character layer.

Recently, end-to-end solvers have been proposed that use a single deep learning model to solve a CAPTCHA directly from the input image without any segmentation. \citet{noury2020deep} proposed a CNN-based solver that has multiple character classification heads, each of which is responsible to predict a character in the CAPTCHA image. However, this model can only predict CAPTCHAs with fixed length. This problem was overcome by \citet{tian2020generic} using a null character class. \citet{li2021end} used a convolutional recurrent neural network to train a solver on cycle-GAN generated data and then employed active transfer learning to optimize it on real schemes using a small labeled dataset. \citet{deng3e} proposed a semi-supervised solver based on the encoder-decoder architecture and attention mechanism which only requires small portion of the training dataset to be labeled. Another recent work by~\citet{ousat2024matter} used an object detection model to localize and predict each character in a given CAPTCHA simultaneously.

In this paper, we not only propose a new method for solving text-based CAPTCHAs but also develop a rigorous mathematical theory for this important problem. To the best of our knowledge, the only previous work that provided some theoretical analyses for CAPTCHA systems is~\citet{li2018captcha}, where a game-theoretical approach was used to analyze the interactions between the defender and the attacker as well as the benefits of human solvers alongside machine solvers. Our work significantly differs from theirs since we only consider machine solvers and prove the effectiveness of these solvers using probability theory. Our theoretical results are the first of this kind for CAPTCHA solvers.

\minisection{Uncertainty Estimation and Out-of-distribution Detection.} 
Our paper is related to uncertainty estimation and out-of-distribution detection~\citep{abdar2021review}. Here we review only the latest related work for deep learning models. Previous works~\citep{guo2017calibration, pereyra2017regularizing} have shown that modern deep learning models often do not provide well-calibrated uncertainty, in the sense that they tend to make incorrect predictions on out-of-distribution data with high confidence. Several lines of work have been proposed to improve uncertainty estimation for deep learning. For example, Bayesian neural networks~\citep{neal1995bayesian} provide a principled framework for studying uncertainty in deep learning. However, exact Bayesian inference is intractable for modern deep learning architectures and approximate inference is usually required~\citep{chen2014stochastic, blundell2015weight, gal2016dropout, ritter2018scalable, zhang2020csgmcmc, rudner2021tractable}. Popular approximate inference methods include variational inference~\citep{blundell2015weight, rudner2021tractable}, Markov chain Monte Carlo methods~\citep{chen2014stochastic, zhang2020csgmcmc}, and Laplace approximation~\citep{ritter2018scalable}. Besides Bayesian approaches, Monte Carlo dropout~\citep{gal2016dropout} and deep ensembles~\citep{lakshminarayanan2017simple} are also commonly used for deep learning uncertainty quantification. Once a good uncertainty estimation is obtained, it can be straightforwardly used for out-of-distribution detection~\citep{lakshminarayanan2017simple, hendrycks2017baseline}.

Among these methods, deep ensembles~\citep{lakshminarayanan2017simple} provide a simple way for uncertainty estimation that can be considered an approximation of Bayesian model averaging~\citep{wilson2020bayesian}. Its main idea is to train multiple neural networks (called base models) and use them for inference. Training a base model is the same as training an ordinary neural network and the model diversity is ensured by random parameter initialization and shuffling of the training dataset for each base model. This phenomenon was explained by \citet{fort2019deep} using a loss landscape perspective. \citet{d2021repulsive} later proposed an improvement to deep ensembles by introducing a kernelized repulsive term in the update rule to ensure model diversity when the number of parameters is large and the effect of random initialization reduces.

Our work develops a novel theoretical analysis for ensemble methods. Previous work on ensemble theory mainly focused on general ensembles and derived risk bounds for the weighted majority vote classification method~\citep{germain2015risk, laviolette2017risk, masegosa2020second} and the weighted average regression method~\citep{cuong2013generalization, ho2020posterior}. More recent theoretical results also proved the relationship between generalization and diversity of deep ensembles~\citep{ortega2022diversity} and the improvement rates of ensembles over a single model~\citep{theisen2024ensembles}. In contrast, our analysis focuses more on the CAPTCHA solving problem and exploits its structure to derive bounds on the success rates of the solver.

\section{Uncertainty-Aware CAPTCHA Solver Using Deep Ensembles}
\label{sec:EnSolver_LEnSolver}

In this section, we describe our proposed uncertainty-aware CAPTCHA solvers in detail. Our first solver, named \textbf{EnSolver}, uses an ensemble of deep learning-based CAPTCHA solvers to make a prediction along with a corresponding uncertainty estimate on a given input text-based CAPTCHA image. The uncertainty estimates allow EnSolver to detect inputs dissimilar to the training data (i.e., out-of-distribution inputs) and thus can ``skip'' inputs that are hard to crack. 
Detecting out-of-distribution CAPTCHAs is an important feature in the solving process because web applications often define policies on the number of failed attempts before locking an account or injecting a long delay before showing the next CAPTCHA. Consequently, if a trained solver is equipped with a pre-filtering mechanism that can effectively detect and skip CAPTCHAs that are not likely to be solved correctly, it can effectively bypass these account lockout policies and avoid triggering subsequent access failures used to lock out web sessions. We also extend EnSolver to \textbf{LEnSolver}, a new solver for the more practical setting where only a limited number of skips are allowed. LEnSolver also employs the same mechanism as EnSolver to skip uncertain inputs, but it will be forced to make a prediction once the maximum allowable number of skips is reached.

\subsection{Uncertainty-Aware CAPTCHA Solver}

\begin{algorithm}[tb]
\caption{Generic Uncertainty-Aware CAPTCHA Solver}\label{alg:overall}
\begin{algorithmic}
    \Require Trained model $m$, input image $x$, threshold $\tau$
    \State $(y, u) \gets \Call{predict\_with\_uncertainty}{m, x}$
    \If{$u < \tau$}
        \State Predict with $y$
    \Else
        \State Skip $x$
    \EndIf
\end{algorithmic}
\end{algorithm}

\begin{algorithm}[tb]
\caption{Training a Deep Ensemble}\label{alg:de_train}
\begin{algorithmic}
    \Require Labeled training dataset $\mathcal{D}$, number of base models $M$
    \ForAll{$i \in \{1, 2, \ldots, M\}$}
        \State Initialize base model $m_i$ randomly
        \State Train $m_i$ using $\mathcal{D}$
    \EndFor
    \State $m \gets (m_1, m_2, \ldots, m_M)$
    \State \Return $m$
\end{algorithmic}
\end{algorithm}

\begin{algorithm}[tb]
\caption{Uncertainty Estimation with Deep Ensemble}
\label{alg:de_uncertainty}
\begin{algorithmic}
    \Require Ensemble model $m = (m_1, m_2, \ldots, m_M)$, input image $x$
    \Function{predict\_with\_uncertainty}{$m, x$}
        \ForAll{$i \in \{1, 2, \ldots, M\}$}
            \State $y_i \gets \text{predict}(m_i, x)$
        \EndFor
        \State $(s_1, s_2, \ldots, s_N) \gets \text{unique}(y_1, y_2, \ldots, y_M)$
        \ForAll{$j \in \{1, 2, \ldots, N\}$}
            \State $\displaystyle p_j \gets \frac{| \left\{ i \in \{ 1, 2, \ldots, M \} |y_i = s_j \right\} |}{M}$
        \EndFor
        \State $p_{\textrm{max}} \gets \max_{j=1, \ldots, N} \{ p_j \}$
        \State $j_{\textrm{max}} \gets \argmax_{j=1, \ldots, N} \{ p_j \}$
        \State $u \gets 1 - p_{\textrm{max}}$
        \State $y \gets s_{j_{\textrm{max}}}$
        \State \Return $(y, u)$
    \EndFunction
\end{algorithmic}
\end{algorithm}

For a conventional deep learning-based solver, a deep learning model $m$ is trained and then used to make a prediction $y = m(x)$ if given an input CAPTCHA image $x$. In this case, $y$ is the string that the model $m$ returns when the input image $x$ is fed into the model. Several types of deep learning models have been proposed for this problem that include generative adversarial networks~\citep{ye2018yet}, convolutional neural networks~\citep{noury2020deep}, attention-based encoder-decoder networks~\citep{deng3e}, and object detection-based models~\citep{ousat2024matter}.

Although experimentally accurate, these deep learning-based solvers falter on images with visual characteristics (e.g., text font or character size) unseen in the training data. When encountering these out-of-distribution samples, it is better to not give an answer than to give wrong answers and get locked out by the defense system. Additionally, most CAPTCHA systems allow users to request a new image, thus skipping an answer enables the solver to exploit this feature to switch to an easier image.

To allow this new capability in a CAPTCHA solver, we propose a simple generic uncertainty-aware CAPTCHA solver in Algorithm~\ref{alg:overall}. Our solver is equipped with a trained uncertainty-aware model $m$ and a real-valued threshold $\tau \in (0, 1]$. When given an input CAPTCHA image $x$, the model $m$ first predicts the string $y$ together with an uncertainty level $u \in [0, 1]$ via the function call $\Call{predict\_with\_uncertainty}{m, x}$. The uncertainty level $u$ is a real number indicating the extent to which the model $m$ is \emph{unconfident} (or \emph{uncertain}) that $y$ is the text string shown in the input image $x$. Note that the higher the uncertainty level $u$, the less likely that the prediction $y$ is correct. With the uncertainty level $u$, the next step is to compare it with the threshold $\tau$. If $u$ is below this threshold, the solver will return the text string $y$ as the answer for the CAPTCHA. Otherwise, it will skip $x$ and request a new CAPTCHA image if possible.

Note that for our uncertainty-aware solver above, the range of $u$ and the choice of $\tau$ depend on the specific uncertainty estimation method. Since our method relies on the uncertainty level $u$ predicted by the uncertainty-aware model $m$, a model with a good uncertainty estimation capability is essential for our method to work well. In the next section, we shall describe EnSolver, an instance of the above generic solver that uses a deep ensemble~\citep{lakshminarayanan2017simple, d2021repulsive} to obtain the uncertainty level $u$.

\subsection{EnSolver: The Deep Ensemble Solver}
\label{sec:ensolver}

EnSolver is an uncertainty-aware CAPTCHA solver that employs a deep ensemble of base CAPTCHA solvers for uncertainty estimation. Deep ensemble \citep{lakshminarayanan2017simple, d2021repulsive} is a popular uncertainty estimation approach that requires significantly less modifications to the original model architecture as well as the training and inference pipelines, as compared to other uncertainty estimation approaches such as Bayesian methods \citep{chen2014stochastic, blundell2015weight, gal2016dropout, zhang2020csgmcmc, rudner2021tractable}.

We choose deep ensemble for uncertainty estimation since it allows our approach to have a greater applicability, especially to solvers with a very complex model architecture, while not compromising the quality of the uncertainty estimates. In our experiments, we will demonstrate our approach for two such complex models, one that uses the encoder-decoder architecture combined with the attention mechanism and another that uses object detectors. For these complex models, using a Bayesian method for uncertainty estimation is unnecessarily hard and inefficient since it requires a major modification to the model architecture and training procedure \citep{blundell2015weight, gal2016dropout, ritter2018scalable, zhang2020csgmcmc}. Another reason that we choose deep ensemble for uncertainty estimation is that it allows us to develop theoretical guarantees for our solvers (see Section~\ref{sec:theory}). We must emphasize that despite the simplicity of our approach, we can achieve very high accuracy, as we will show in our experiments in Section~\ref{sec:experiment}. This is consistent with several previous works \citep{ovadia2019can, wilson2020bayesian} that showed the competitiveness of deep ensembles compared to Bayesian methods such as variational inference or Laplace approximation.

For EnSolver, the model $m$ in Algorithm~\ref{alg:overall} is an ensemble of $M$ base models $(m_1, m_2, \ldots, \allowbreak m_M)$, each of which is a conventional CAPTCHA solver. The uncertainty level $u$ for this solver will be estimated from the agreement among the base models on a given input image $x$. Throughout this paper, we will use $m$ to denote both the ensemble and the uncertainty-aware solver itself, while we use $m_i$ to denote each base solver in the ensemble. We now describe how to train this ensemble of models and how to use it for uncertainty estimation.

\minisection{Training.} Given the number of base models $M$ and a labeled training set $\mathcal{D}$, we follow~\citet{lakshminarayanan2017simple} and build the deep ensemble by training $M$ base models separately with different random initializations. This training process is illustrated in Algorithm~\ref{alg:de_train}. In general, each base model $m_i$ may have its own architecture and can be trained with its own training process. However, in most cases in practice~\citep{blundell2015weight, d2021repulsive}, we only need to use one architecture and one training process (e.g., stochastic gradient descent with cross entropy loss). To speed up the training process, we can also train the base models in parallel using different GPUs~\citep{lakshminarayanan2017simple}. Our final ensemble is the set of $M$ well-trained base models $m = (m_1, m_2, \ldots, m_M)$.

\minisection{Uncertainty Estimation.} Given a trained ensemble model $m = (m_1, m_2, \ldots, m_M)$ and any input image $x$, we compute the predicted output string $y$ and the uncertainty estimate $u$ using the $\Call{predict\_with\_uncertainty}{m, x}$ function in Algorithm \ref{alg:de_uncertainty}. More specifically, we first use each base model $m_i$ to predict a string $y_i$ on the input image $x$. Since some models may predict a similar string, especially when the string is the correct prediction, we will obtain a set $(s_1, s_2, \ldots, s_N)$ of $N$ distinct strings among the $M$ predictions. Note that the set of predictions $(y_1, y_2, \ldots, y_M)$ imposes a predictive distribution on $(s_1, s_2, \ldots, s_N)$ where the probability $p_j$ of $s_j$ is:
\begin{equation*}
p_j = \frac{| \left\{ i \in \{ 1, 2, \ldots, M \} |y_i = s_j \right\} |}{M},
\end{equation*}
which is the proportion of the number of base models that predict $s_j$. In this distribution, $p_{\textrm{max}} = \max_j p_j$ measures the agreement among the predictions of our base models. It has a high value when a large proportion of the base models give the same prediction, i.e., when the uncertainty is low. Thus, we can use $u = 1 - p_{\textrm{max}}$ as the quantification of the predictive uncertainty. Note that the value of $u$ ranges from 0 to $1-\frac{1}{M}$. The predicted string $y$ returned from our procedure is the string that has the maximum number of predictions, i.e., the one corresponding to $p_{\textrm{max}}$. An illustration of this uncertainty estimation method is depicted in Figure~\ref{fig:ensolver}.

\subsection{LEnSolver: The Limited-Skip EnSolver}
\label{sec:limskips}

Although the EnSolver in Section~\ref{sec:ensolver} can skip predictions on uncertain inputs, it would be impractical if the solver keeps skipping and does not make any prediction within a reasonable time limit. Since the purpose of a CAPTCHA solver is to gain access to a website or a system, we expect the solver to finally make a prediction when the time limit is reached regardless of the uncertainty level. Thus, for this purpose, we extend EnSolver to LEnSolver, which can be applied to the setting where we only allow a maximum number of skips.

The LEnSolver is described in Algorithm~\ref{alg:limskips}. This algorithm requires the original EnSolver $m$, which is obtained from Algorithm~\ref{alg:overall} with the uncertainty estimate in Algorithm~\ref{alg:de_uncertainty}, together with the maximum number of skips $T$. For the first $T$ iterations of the algorithm, we use $m$ to make a prediction on the current input image $x_t$ and allow the solver to skip to the next input if the uncertainty level is high. However, if the solver skips all the first $T$ inputs, then it is forced to make a prediction on the $(T+1)^{th}$ input, $x_{T+1}$, regardless of the uncertainty level. This prediction also uses the majority prediction from the base models that can be obtained by calling $\Call{predict\_with\_uncertainty}{m, x_{T+1}}$. We will write $m^T$ to denote this LEnSolver. In the next section, we will prove theoretical guarantees for both EnSolver and LEnSolver.

\begin{algorithm}[tb]
\caption{LEnSolver: Limited-Skip EnSolver}
\label{alg:limskips}
\begin{algorithmic}
    \Require Original EnSolver $m$, maximum number of skips $T$
    \ForAll{$t \in \{1, 2, \ldots, T\}$}
        \State Receive a random input image $x_t$
        \State $y_t \gets \text{predict}(m, x_t)$
        \If{$y_t \neq \text{skip}$}
            \State Return $y_t$
        \EndIf
    \EndFor
    \State Receive a random input $x_{T+1}$
    \State $(y, u) \gets \Call{predict\_with\_uncertainty}{m, x_{T+1}}$
    \State Return $y$
\end{algorithmic}
\end{algorithm}

\section{Theoretical Guarantees for EnSolver and LEnSolver}
\label{sec:theory}

The use of ensembles for uncertainty estimation in EnSolver and LEnSolver allows us to derive some novel theoretical properties for these solvers. In particular, we will establish in this section theoretical lower bounds for the right decision rate of EnSolver and the success rate of LEnSolver that can be used to explain the effectiveness of these solvers in practice. To the best of our knowledge, our paper is the first to provide such theoretical analyses for CAPTCHA solvers. In Section~\ref{sec:theory-setting}, we will describe the mathematical settings for our theory. Sections~\ref{sec:theory-ensolver} and~\ref{sec:theory-lensolver} then show the theoretical guarantees for EnSolver and LEnSolver respectively. Empirical validations of our theory are provided in Section~\ref{sec:experiment}.

\subsection{Mathematical Settings}
\label{sec:theory-setting}

Let $\mathcal{X}$ be the input domain, which is the set of all CAPTCHA images containing at most $\ell$ characters, for some integer $\ell > 0$. Let $\mathcal{S}$ be the output domain, which is the set of all possible output strings containing at most $\ell$ characters from a finite alphabet. Note that the cardinality of $\mathcal{X}$ could be infinity since $\mathcal{X}$ contains images; however, $\mathcal{S}$ is a finite set given a finite alphabet and a fixed $\ell$. A data point is a pair $(x, s_x) \in \mathcal{X} \times \mathcal{S}$, where $x \in \mathcal{X}$ is an input CAPTCHA image and $s_x \in \mathcal{S}$ is its corresponding output string. We assume the input images have a probability distribution $p(x)$ on $\mathcal{X}$ and each input $x$ has a single correct output $s_x$. We will consider the case where $\mathcal{X} = \mathcal{X}^{\text{in}} \, \cup \, \mathcal{X}^{\text{out}}$, with $\mathcal{X}^{\text{in}}$ containing in-distribution images, $\mathcal{X}^{\text{out}}$ containing out-of-distribution images, and ${ \mathcal{X}^{\text{in}} \, \cap \, \mathcal{X}^{\text{out}} = \emptyset}$. Let $\alpha$ be the proportion of the in-distribution data in the whole input domain. We have:
\begin{equation*} 
\alpha = p(x \in \mathcal{X}^{\text{in}}) = \int_{x \in \mathcal{X}^{\text{in}}} p(x) dx.
\end{equation*}

Consider the EnSolver in Section~\ref{sec:ensolver}. In this solver, we have an ensemble model ${ m = (m_1, m_2, \ldots, m_M) }$ that contains $M$ base models. Each base model $m_i$ is a function mapping $x \in \mathcal{X}$ to $m_i(x) \in \mathcal{S}$. Let $\beta_i$ denote the in-distribution accuracy of $m_i$. We have:
\begin{equation*}
\beta_i = \mathbb{E}_{x \sim p( \,.\, | \, x \in \mathcal{X}^{\text{in}})} \left[ \mathbf{1}(m_i(x) = s_x) \right] = p(m_i(x) = s_x \,|\, x \in \mathcal{X}^{\text{in}}),
\end{equation*}
where $\mathbf{1}(\cdot)$ is the indicator function.
We also let $\beta_{\text{min}}$ and $\beta_{\text{max}}$ be respectively the minimum and maximum in-distribution accuracies among the base models. That is:
\begin{align*}
\beta_{\text{min}} &= \min_{i \in \{ 1, \ldots, M \}} \beta_i, \text{ and} \\
\beta_{\text{max}} &= \max_{i \in \{ 1, \ldots, M \}} \beta_i.
\end{align*}

Since the behavior of EnSolver does not change when we adjust $\tau$ to $\tau'$ such that $M(1-\tau') = \lfloor M(1-\tau) \rfloor$, we can assume $M(1-\tau)$ is a positive integer without loss of generality. Following the convention, we assume each base model $m_i$ makes a prediction independently with each other given an input. To derive our theoretical results, we also make the following assumptions.
\begin{enumerate}
\item[(A1)] $\beta_{\text{min}} > 1/N_\mathcal{S}$, where $N_\mathcal{S}$ is the size of the output domain $\mathcal{S}$.
\item[(A2)] For any base model $m_i$ and $x\in \mathcal{X}^{\text{in}}$, the model $m_i$ predicts any incorrect string with the same probability. This assumption implies ${ p(m_i(x) = s \,|\, x\in \mathcal{X}^{\text{in}}) = (1-\beta_i)/(N_\mathcal{S}-1) }$ for all $s \ne s_x$.
\item[(A3)] For any base model $m_i$ and $x\in \mathcal{X}^{\text{out}}$, the base model $m_i$ predicts any string in $\mathcal{S}$ with the same probability. This assumption implies ${ p(m_i(x) = s \,|\, x \in \mathcal{X}^{\text{out}}) = 1/N_\mathcal{S} }$ for all $s \in \mathcal{S}$.
\end{enumerate}

We note that $N_\mathcal{S}$ is usually very large in practice, e.g., $N_\mathcal{S} = 26^5$ for a typical CAPTCHA system with CAPTCHA strings of length $5$ from an alphabet of size $26$. Thus, (A1) is a very mild assumption if all the base models are trained reasonably well. Furthermore, due to the large $N_\mathcal{S}$, the uniform assumptions in (A2) and (A3) are also reasonable for deriving useful bounds, although they may not hold in practice. In Section~\ref{sec:experiment}, we will conduct experiments to show the usefulness of our theory even with these assumptions. Now we are ready to discuss our main theoretical results.

\subsection{Theoretical Guarantee for EnSolver}
\label{sec:theory-ensolver}

We first focus on the EnSolver described in Section~\ref{sec:ensolver}, which is an instance of Algorithm~\ref{alg:overall} with the uncertainty estimation method in Algorithm~\ref{alg:de_uncertainty}. Recall that this solver can either predict an output string or skip the prediction given an input image. To measure the performance of this solver, we will analyze its right decision rate, which is defined below.

\begin{defn}
For any EnSolver $m$ and any input $x \in \mathcal{X}$, let $m(x) \in \mathcal{S} \cup \{ s_{\text{skip}} \}$ be the output of this EnSolver, where $s_{\text{skip}}$ is a special symbol indicating that the solver will skip predicting $x$. The \textbf{right decision rate} $R_{\text{rd}}(m)$ of the EnSolver $m$ is defined as:
\begin{equation*}
R_{\text{rd}}(m) =
p \left(m(x) = s_x, x \in \mathcal{X}^{\text{in}} \right) + p \left(m(x) \in \{s_x, s_{\text{skip}}\}, x \in \mathcal{X}^{\text{out}} \right),
\end{equation*}
which is the total probability that this solver would make a correct decision on the whole input domain.
\label{def:rdr}
\end{defn}

Note that this definition captures the desired behavior of an EnSolver: on an in-distribution input, we want the model to make the correct prediction, while it can either skip or make the correct prediction on an out-of-distribution input. This definition is also equivalent to the expectation that the EnSolver $m$ will make a correct decision on a random input $x \sim p(x)$. In this section, we will give a lower bound for the right decision rate $R_{\text{rd}}(m)$ of any EnSolver $m$. Our bound will depend on a novel quantity called the out-of-distribution error bound, which we define as follows.

\begin{defn}
The \textbf{out-of-distribution error bound} (OEB) of an EnSolver with ensemble size $M$, output domain size $N_{\mathcal{S}}$, and uncertainty threshold $\tau$, is defined as:
\begin{equation*}
\mathcal{E}(M, N_{\mathcal{S}}, \tau) = \begin{pmatrix}M \\ \lfloor M/2 \rfloor \end{pmatrix}\frac{1}{(N_{\mathcal{S}})^{M(1-\tau)}},
\end{equation*}
where $\begin{pmatrix}M \\ \lfloor M/2 \rfloor \end{pmatrix}$ is the binomial coefficient ``$M$ choose $\lfloor M/2 \rfloor$''.
\label{def:ceb}
\end{defn}

In practice, the OEB is usually very small. For instance, if we consider the CAPTCHA system above with output strings of length 5 and an alphabet of size 26, an EnSolver with ensemble size $M=10$ and uncertainty threshold $\tau = 0.5$ will have the OEB value ${ \mathcal{E}(10, 26^5, 0.5) = 252/26^{25} \approx 10^{-33} }$. Furthermore, as its name suggests, the OEB is an upper bound of the EnSolver's error rate on out-of-distribution data. This result is stated in Lemma~\ref{lem:oeb_out} below, with the proof given in Appendix~\ref{proof:oeb_out}.

\begin{lem}
For any EnSolver $m$ with ensemble size $M$, output domain size $N_{\mathcal{S}}$, and uncertainty threshold $\tau$, its OEB satisfies:
\begin{align*}
    \mathcal{E}(M, N_{\mathcal{S}}, \tau) > p \left( m(x) \notin \{ s_x, s_{\text{skip}} \} \,|\, x \in \mathcal{X}^{\text{out}} \right).
\end{align*}
\label{lem:oeb_out}
\end{lem}

Given our observation that $\mathcal{E}(M, N_{\mathcal{S}}, \tau)$ is usually very small, this lemma tells us that the error rate of the EnSolver on out-of-distribution data is also very small. Interestingly, $\mathcal{E}(M, N_{\mathcal{S}}, \tau)$ also gives us an upper bound for the probability that several base models in the ensemble make the same but incorrect prediction on in-distribution data. This result is stated in the following lemma, where the proof is given in Appendix~\ref{proof:oeb_in}.

\begin{lem}
Consider any EnSolver $m$ with ensemble size $M$, output domain size $N_{\mathcal{S}}$, and uncertainty threshold $\tau$. For any $x \in \mathcal{X}$ and $s \in \mathcal{S}$, let $n(x, s)$ be the number of base models predicting the output $s$ for $x$. The OEB of $m$ satisfies:
\begin{equation*}
    \mathcal{E}(M, N_{\mathcal{S}}, \tau) > p \left( \exists s \ne s_x, ~ n(x, s) \ge M(1-\tau) + 1 \,|\, x \in \mathcal{X}^{\text{in}} \right).
\end{equation*}
\label{lem:oeb_in}
\end{lem}

Since $\, p \left( \exists s \ne s_x, ~ n(x, s) \ge M(1-\tau) + 1 \,|\, x \in \mathcal{X}^{\text{in}} \right) \ge p \left( \exists s \ne s_x, m(x) = s \,|\, x \in \mathcal{X}^{\text{in}} \right)$, \allowbreak Lemma~\ref{lem:oeb_in} implies that $\mathcal{E}(M, N_{\mathcal{S}}, \tau)$ is also an upper bound of the EnSolver's wrong prediction rate on in-distribution data.\footnote{Here we differentiate between the wrong prediction rate, which is $p \left( \exists s \ne s_x, m(x) = s \,|\, x \in \mathcal{X}^{\text{in}} \right)$, and the error rate, which is $p \left( m(x) \neq s_x \,|\, x \in \mathcal{X}^{\text{in}} \right) = p \left( m(x) = s_{\text{skip}} \vee \exists s \ne s_x, m(x) = s \,|\, x \in \mathcal{X}^{\text{in}} \right)$, given an in-distribution input.} It also tells us that this wrong prediction rate on in-distribution data is usually small. Using both Lemma~\ref{lem:oeb_out} and Lemma~\ref{lem:oeb_in}, we can now prove a lower bound for the right decision rate $R_{\text{rd}}(m)$ in Theorem~\ref{thrm:rdr} below. The proof for this theorem is given in Appendix~\ref{proof:rdr}.

\begin{thrm}
Consider any EnSolver $m$ with ensemble size $M$, output domain size $N_{\mathcal{S}}$, and uncertainty threshold $\tau$. Let $\beta_{\text{min}}$ and $\beta_{\text{max}}$ be respectively the minimum and maximum in-distribution accuracies among the $M$ base models. The right decision rate of $m$ satisfies:
\begin{equation}
    R_{\text{rd}}(m) > \alpha \sum_{i = k}^{M}\begin{pmatrix} M \\ i \end{pmatrix}\beta_{\text{min}}^i \, (1-\beta_{\text{max}})^{M-i} + (1-\alpha) - \mathcal{E}(M, N_{\mathcal{S}}, \tau),
\label{eq:rdr_guarantee}
\end{equation}
where $k = M(1-\tau) + 1$ and $\alpha$ is the proportion of in-distribution data in the input domain.
\label{thrm:rdr}
\end{thrm}

This theorem gives us a theoretical guarantee for the right decision rate of an EnSolver. From the theorem, we see that if $\alpha$ is small (i.e., there are more out-of-distribution data), the lower bound will be dominated by $(1-\alpha)$, which will be large. On the other hand, if there are more in-distribution data (i.e., $\alpha$ is large), then the lower bound will be dominated by the sum on the right-hand side of~\eqref{eq:rdr_guarantee}, which depends on the in-distribution accuracies of the base models in the ensemble.

As an example, let us consider again the CAPTCHA system described previously. Recall that for this system, $M = 10$, $\tau = 0.5$, and $\mathcal{E}(M, N_{\mathcal{S}}, \tau) \approx 10^{-33}$. If the base models were trained so that $\beta_{\text{min}} \approx \beta_{\text{max}} \approx 0.9$, Theorem~\ref{thrm:rdr} would give us the lower bounds $R_{\text{rd}}(m) \gtrapprox 0.9998$ for $\alpha = 0.1$, $R_{\text{rd}}(m) \gtrapprox 0.9991$ for $\alpha = 0.5$, and $R_{\text{rd}}(m) \gtrapprox 0.9985$ for $\alpha = 0.9$. These bounds indicate that the EnSolver will have very high right decision rates in these settings. In Section~\ref{sec:experiment}, we will conduct experiments to validate our theoretical bounds on real data.

\subsection{Theoretical Guarantee for LEnSolver}
\label{sec:theory-lensolver}

We now focus on the LEnSolver described in Section~\ref{sec:limskips} and Algorithm~\ref{alg:limskips}. This solver always makes a prediction for an input image; thus, we will measure its performance through its success rate, which is defined below.

\begin{defn}
Consider an LEnSolver $m^T$ constructed from an original EnSolver $m$ with the maximum number of skips $T$. The \textbf{success rate} $R_{\text{s}}(m^T)$ of the LEnSolver $m^T$ is defined as the probability that $m^T$ successfully cracks the CAPTCHA system.
\label{def:success_rate}
\end{defn}

From this definition, we can write $R_{\text{s}}(m^T)$ formally by considering $(T+1)$ input images $x_1, x_2, \ldots, x_{T+1}$ that are sampled i.i.d. from $p(x)$. In this case,
\begin{align}
R_{\text{s}}(m^T) &= p \left( m(x_1) = s_{x_1} \right) + p \left( m(x_1) = s_{\text{skip}}, \, m(x_2) = s_{x_2} \right) + \ldots \notag \\
&\quad + p \left( m(x_1) = s_{\text{skip}}, \, m(x_2) = s_{\text{skip}}, \ldots, \, m(x_T) = s_{\text{skip}}, \, m(x_{T+1}) = s_{x_{T+1}} \right) \notag \\
&= \sum_{t=1}^{T+1} \left( p \left( m(x_t) = s_{x_t} \right) \prod_{i=1}^{t-1} p \left( m(x_i) = s_{\text{skip}} \right) \right).
\label{eq:success_rate}
\end{align}

Note that when actually running Algorithm~\ref{alg:limskips}, we do not need to sample all these $(T+1)$ input images if the solver decides to make a prediction before exhausting all the $T$ allowable skips. However, we need to consider all these inputs to give a formal definition for $R_{\text{s}}(m^T)$. Since $x_1, x_2, \ldots, x_{T+1}$ are i.i.d., Equation~\eqref{eq:success_rate} can be further simplified by:
\begin{align}
R_{\text{s}}(m^T) &= p \left( m(x) = s_x \right) \sum_{t=0}^T p \left( m(x) = s_{\text{skip}} \right)^t \label{eq:success_rate_1} \\
&= p \left( m(x) = s_x \right) \frac{1 - p \left( m(x) = s_{\text{skip}} \right)^{T+1}}{1 - p \left( m(x) = s_{\text{skip}} \right)}, \label{eq:success_rate_2}
\end{align}
where $x$ is a random input sample from $p(x)$. From Equation~\eqref{eq:success_rate_2}, we see that the success rate $R_{\text{s}}(m^T)$ depends on the correct prediction rate $p \left( m(x) = s_x \right)$ and the skip rate $p \left( m(x) = s_{\text{skip}} \right)$ of any random input $x$. We also observe that $R_{\text{s}}(m^T)$ is a monotonically increasing function of $T$ if $p \left( m(x) = s_{\text{skip}} \right) > 0$. Furthermore, when $T \rightarrow \infty$ (i.e., we are allowed to skip infinitely), we have:
\begin{align*}
R_{\text{s}}(m^T) \rightarrow \frac{p ( m(x) = s_x )}{1 - p ( m(x) = s_{\text{skip}} )} = p( m(x) = s_x \,|\, m(x) \neq s_{\text{skip}}).
\end{align*}

Note that $R_{\text{s}}(m^T)$ does not tend to $1$ since we can only make one single prediction. We also observe from Equation~\eqref{eq:success_rate_1} that $R_{\text{s}}(m^T)$ is high when both $p ( m(x) = s_x )$ and $p ( m(x) = s_{\text{skip}} ) $ are high. However, since 
$p ( m(x) = s_{\text{skip}} ) \le 1 - p ( m(x) = s_x )$, it is usually the case that when $p ( m(x) = s_x )$ is high, $p ( m(x) = s_{\text{skip}} )$ would be small and vice versa.

In this section, we will give a lower bound for the success rate $R_{\text{s}}(m^T)$ of any LEnSolver $m^T$. From Equations~\eqref{eq:success_rate_1} and~\eqref{eq:success_rate_2}, it is sufficient to derive lower bounds for $p ( m(x) = s_x )$ and $p ( m(x) = s_{\text{skip}} ) $ and then combine them into a lower bound for $R_{\text{s}}(m^T)$. In Lemmas~\ref{lem:correct_rate} and~\ref{lem:skip_rate} below, we give the lower bounds for these probabilities, with their proofs given in Appendices~\ref{proof:correct_rate} and~\ref{proof:skip_rate} respectively. Note that the results in these lemmas are for the original EnSolver $m$ of the LEnSolver $m^T$.

\begin{lem}
Consider any EnSolver $m$ with ensemble size $M$, output domain size $N_{\mathcal{S}}$, and uncertainty threshold $\tau$. Let $\beta_{\text{min}}$ and $\beta_{\text{max}}$ be respectively the minimum and maximum in-distribution accuracies among the $M$ base models. The correct prediction rate of $m$ satisfies:
\begin{equation*}
   p \left( m(x) = s_x \right) > \alpha \sum_{i=k}^{M}\begin{pmatrix} M \\ i \end{pmatrix}\beta_{\text{min}}^i (1-\beta_{\text{max}})^{M-i} - \alpha \, \mathcal{E}(M, N_{\mathcal{S}}, \tau),
\end{equation*}
where $k = M(1-\tau) + 1$ and $\alpha$ is the proportion of in-distribution data in the input domain.
\label{lem:correct_rate}
\end{lem}

\begin{lem}
Consider any EnSolver $m$ with ensemble size $M$, output domain size $N_{\mathcal{S}}$, and uncertainty threshold $\tau$. Let $\beta_{\text{min}}$ and $\beta_{\text{max}}$ be respectively the minimum and maximum in-distribution accuracies among the $M$ base models. The skip rate of $m$ satisfies:
\begin{equation*}
   p \left( m(x) = s_{\text{skip}} \right) > \alpha \sum_{i=0}^{k-1}\begin{pmatrix} M \\ i \end{pmatrix}\beta_{\text{min}}^i (1-\beta_{\text{max}})^{M-i} + (1-\alpha) - \frac{N_{\mathcal{S}}-\alpha}{N_{\mathcal{S}}-1} \, \mathcal{E}(M, N_{\mathcal{S}}, \tau),
\end{equation*}
where $k = M(1-\tau) + 1$ and $\alpha$ is the proportion of in-distribution data in the input domain.
\label{lem:skip_rate}
\end{lem}

Combining Lemmas~\ref{lem:correct_rate} and~\ref{lem:skip_rate} with Equations~\eqref{eq:success_rate_1} and~\eqref{eq:success_rate_2}, it is straightforward to show Theorem~\ref{thrm:success} below that gives a lower bound for the success rate $R_{\text{s}}(m^T)$ of an LEnSolver.

\begin{thrm}
Consider any LEnSolver $m^T$ with an original EnSolver $m$, the maximum number of skips $T$, ensemble size $M$, output domain size $N_{\mathcal{S}}$, and uncertainty threshold $\tau$. Let $\beta_{\text{min}}$ and $\beta_{\text{max}}$ be respectively the minimum and maximum in-distribution accuracies among the $M$ base models. The success rate of $m^T$ satisfies:
\begin{equation*}
    R_{\text{s}}(m^T) > \gamma \, \frac{1-\rho^{T+1}}{1-\rho},
\end{equation*}
where $\displaystyle \gamma = \alpha \sum_{i=k}^{M}\begin{pmatrix} M \\ i \end{pmatrix}\beta_{\text{min}}^i (1-\beta_{\text{max}})^{M-i} - \alpha \, \mathcal{E}(M, N_{\mathcal{S}}, \tau)$, \\[3pt]
${\hskip 1.1cm} \displaystyle \rho = \alpha \sum_{i=0}^{k-1}\begin{pmatrix} M \\ i \end{pmatrix}\beta_{\text{min}}^i (1-\beta_{\text{max}})^{M-i} + (1-\alpha) - \frac{N_{\mathcal{S}}-\alpha}{N_{\mathcal{S}}-1} \, \mathcal{E}(M, N_{\mathcal{S}}, \tau),$ \\[8pt]
${\hskip 1.1cm} k = M(1-\tau) + 1$, and $\alpha$ is the proportion of in-distribution data in the input domain.
\label{thrm:success}
\end{thrm}

To give a simple illustration for this theorem, we can consider the example in Section~\ref{sec:theory-ensolver}. If we let $\alpha = 0.5$ (i.e., half of the data are out-of-distribution), the success rate of LEnSolver is at least 0.75 if we allow only one skip. This lower bound will increase to 0.93 if we allow three skips and to 0.98 if we allow five skips. In Section~\ref{sec:experiment}, we will conduct experiments on real data to validate this theorem.

\section{Experiments}
\label{sec:experiment}

In this section, we conduct experiments to evaluate the performance of our approaches as well as the usefulness of our theoretical results in practice. We first describe our experiment settings in Section~\ref{sec:exp-setting}. We then evaluate the right decision rates of EnSolver models in Section~\ref{sec:ensolver_exp} and the success rates of LEnSolver models in Section~\ref{sec:lensolver_exp}. We investigate the effects of the maximum number of skips on the success rates of LEnSolvers in Section~\ref{sec:maxskip_exp}. Finally, we discuss some limitations and potential improvements for our theoretical results in Section~\ref{sec:discuss}.

\begin{figure}[tb]
    \centering
    \includegraphics[width=0.7\textwidth]{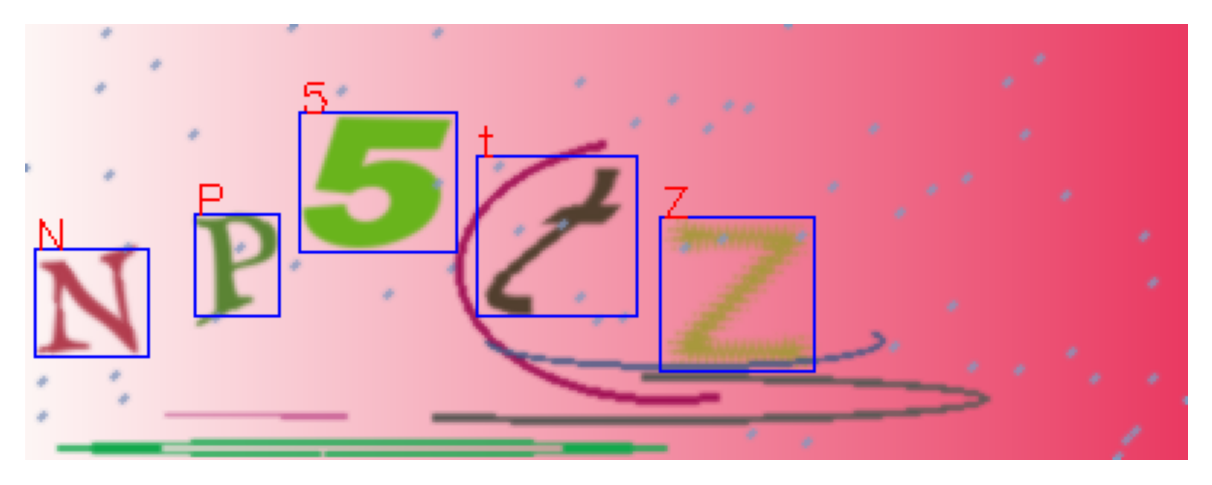}
    \caption{A sample CAPTCHA image in our new dataset. The ground truth label consists of the bounding boxes of each character, the correct letter for each character, and the correct output string NP5tZ.}
    \label{fig:od}
\end{figure}

\subsection{Experiment Settings}
\label{sec:exp-setting}

\minisection{Base Models and Baselines.} We consider two types of base models for our experiments. The first type is the {\bf 3E-Solver} proposed by~\cite{deng3e}, which is a semi-supervised end-to-end solver that uses the encoder-decoder architecture and attention mechanism. The second type of base models used in our experiments is the end-to-end object detection-based solver (named {\bf OD-Solver} in our experiments) proposed by~\cite{ousat2024matter} that uses the YOLOX model~\citep{ge2021yolox} to locate and classify each character in a given input image. These two model types are chosen due to their flexibility and state-of-the-art performance for solving text-based CAPTCHAs. In our experiments, we also use these two models (without ensembling) as our baselines.

\minisection{New In-distribution Dataset.} To train the above object detection models for solving CAPTCHAs, we follow~\citet{ousat2024matter} and generate a new dataset that contains both the bounding box location and the label for each character in the input images. We ensure our dataset is challenging for automatic solvers but still reasonable for human users by carefully generating the background colors, character locations, and noise patterns in the images. The process is as follows.

To generate the background for each image, we first randomize two different RGB triplets for its left and right-most columns. Then we interpolate these two RGB triplets horizontally to obtain the colors at the internal columns while keeping the colors constant on each vertical column. After generating a background image, we randomize a ground truth string containing alphanumeric characters with a random length from 5 to 9. Each of these characters is then converted into an image with random facecolor, font and size, together with a perspective transform to introduce character distortion. Next, the distorted characters are pasted into the background following the ground truth string order with random spacing between them. The position of each character determines its ground truth bounding box, which is combined with the ground truth string and then stored as the ground truth label.

The image obtained from the previous process can already be used as a CAPTCHA, but we can still enhance its resistance against automatic solvers by adding random noises to the image. Specifically, we draw several random curves and dots with various colors into the image where the curves can strike through the characters. The number of curves and dots as well as their thickness are chosen so that they significantly increase the CAPTCHA complexity while keeping the characters on the image recognizable by humans. An example of our generated CAPTCHAs together with the ground truth label are shown in Figure~\ref{fig:od}. For our experiments, we generate 10,000 such CAPTCHA images to use as in-distribution data for model training and another 1,000 images for testing.

\minisection{Out-of-distribution Data.} For out-of-distribution data, we follow~\citet{deng3e} and use publicly available CAPTCHAs from 8 different popular websites (see Figure~\ref{fig:datasets}(b-i) for their names and examples). From these examples, we can see that they have different visual features, background, and noise patterns compared to those of our generated dataset. Thus, it is suitable to treat them as out-of-distribution samples for our experiments. Since the label strings of these public datasets are case-insensitive, we convert all predictions and ground truth label strings to all-lowercase strings before computing the results. For the Microsoft scheme, we concatenate the strings on two lines into a single output string for both prediction and scoring purposes. For each out-of-distribution scheme, we randomly select 1,000 labeled samples that can be used for testing.

\begin{figure}[tb]
    \centering
    \begin{tabular}{ccc}
        \subfloat[Our dataset (in-distribution)]{\includegraphics[width=0.3\textwidth, height=1.7cm]{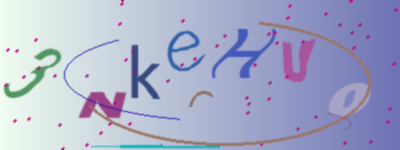}} &
        \subfloat[Apple]{\includegraphics[width=0.3\textwidth, height=1.7cm]{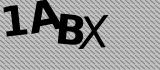}} &
        \subfloat[Ganji]{\includegraphics[width=0.3\textwidth, height=1.7cm]{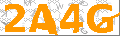}} \\
        \subfloat[Google]{\includegraphics[width=0.3\textwidth, height=1.7cm]{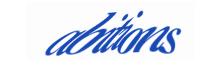}} &
        \subfloat[Microsoft]{\includegraphics[width=0.3\textwidth, height=1.7cm]{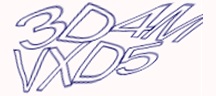}} &
        \subfloat[Sina]{\includegraphics[width=0.3\textwidth, height=1.7cm]{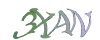}} \\
        \subfloat[Weibo]{\includegraphics[width=0.3\textwidth, height=1.7cm]{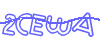}} &
        \subfloat[Wikipedia]{\includegraphics[width=0.3\textwidth, height=1.7cm]{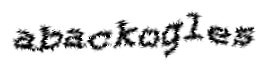}} &
        \subfloat[Yandex]{\includegraphics[width=0.3\textwidth, height=1.7cm]{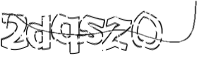}}
    \end{tabular}
    \caption{Our generated dataset (a) and public datasets (b)-(i) used in our experiments. Each dataset has unique visual features.}
    \label{fig:datasets}
\end{figure}

\minisection{Base Models Training.} Since the 3E-Solver was originally developed for semi-supervised learning~\citep{deng3e}, we train each base model of this type using our in-distribution training set and 10 additional unlabeled samples from each out-of-distribution scheme. Thus, each 3E-Solver base model has access to 80 additional unlabeled out-of-distribution samples. We conduct the experiment with the default settings used in the original 3E-Solver paper and train each model for 400 epochs with a learning rate of 0.01. To make each base model applicable to every scheme, we create a unified vocabulary dictionary that includes lowercase and uppercase characters as well as digits to be used for encoding and decoding the texts. This enables us to use the same architecture for different schemes with any set of characters.

For the object detection base models, we implement YOLOX~\citep{ge2021yolox} using the {\tt mmdetection} library~\citep{mmdetection}. The detectors are trained to detect characters from 62 different classes, including 52 English alphabet letters (both lowercase and uppercase) and 10 digits. Each base model is trained for 10 epochs using the Nesterov SGD optimizer \citep{nesterov1983method} with learning rate 0.01 and momentum 0.9. During the training, we save the model checkpoints for every epoch and use the checkpoint with the best mAP on the test set as our base model. Since the out-of-distribution datasets do not have ground truth bounding box information, they cannot be used to train the objection detection models. Thus, we only train the object detection base models using the in-distribution training set.

\minisection{Ensemble Model.} We shall evaluate our methods with ensembles having $M \in \{ 2, 6, 10 \}$ and $\tau \in \{ 0.3, 0.5, 0.7\}$. Note that in order to have non-zero uncertainty, we need at least 2 models in our ensemble. Furthermore, previous work on deep ensembles \citep{lakshminarayanan2017simple} showed that 10 base models are often enough to give a decent uncertainty estimation in most cases while not significantly increasing the time complexity to train the ensemble. The uncertainty threshold values $\tau$  are chosen to give a reasonable balance for our uncertainty estimates.

\subsection{Right Decision Rates of EnSolvers}
\label{sec:ensolver_exp}

\begin{table}[!t]
\centering
\resizebox{0.89\textwidth}{!}{%
\begin{tabular}{ccc cccccc}
\toprule
\multicolumn{3}{c}{Model setting} & \multicolumn{6}{c}{Proportion of in-distribution data} \\
\cmidrule(lr){1-3} \cmidrule(lr){4-9}
Base type & M & $\tau$ & $\alpha=0$ & $\alpha=0.2$ & $\alpha=0.4$ & $\alpha=0.6$ & $\alpha=0.8$ & $\alpha=1$ \\
\midrule
\multirow{16}{*}{3E-Solver} & 1 & & 0.000 & 0.147 & 0.294 & 0.440 & 0.587 & 0.734 \\
& \multicolumn{2}{c}{(baseline)} & (na) & (na) & (na) & (na) & (na) & (na) \\
\cmidrule(lr){2-9}
& 2 & 0.5 & 0.999 & 0.930 & 0.861 & 0.793 & 0.724 & 0.655 \\
& & & (1.000) & (0.908) & (0.816) & (0.723) & (0.631) & (0.539) \\
\cmidrule(lr){2-9}
& 6 & 0.3 & {\bf 1.000}$^*$ & 0.934 & 0.868 & 0.802 & 0.736 & 0.670 \\
& & & (1.000) & (0.883) & (0.766) & (0.650) & (0.533) & (0.416) \\[3pt]
& 6 & 0.5 & {\bf 1.000}$^*$ & 0.945 & 0.890 & 0.836 & 0.781 & 0.726 \\
& & & (1.000) & (0.933) & (0.866) & (0.799) & (0.732) & (0.665) \\[3pt]
& 6 & 0.7 & 0.993 & 0.949 & 0.906 & 0.863 & 0.819 & 0.776 \\
& & & (1.000) & (0.963) & (0.925) & (0.888) & (0.850) & (0.813) \\
\cmidrule(lr){2-9}
& 10 & 0.3 & {\bf 1.000}$^*$ & 0.931 & 0.861 & 0.792 & 0.722 & 0.653 \\
& & & (1.000) & (0.858) & (0.715) & (0.573) & (0.430) & (0.288) \\[3pt]
& 10 & 0.5 & {\bf 1.000}$^*$ & 0.949 & 0.898 & 0.848 & 0.797 & 0.746 \\
& & & (1.000) & (0.908) & (0.816) & (0.724) & (0.632) & (0.541) \\[3pt]
& 10 & 0.7 & {\bf 1.000}$^*$ & 0.958$^*$ & 0.915$^*$ & 0.873$^*$ & 0.830$^*$ & 0.788$^*$ \\
& & & (1.000) & (0.919) & (0.838) & (0.757) & (0.677) & (0.596) \\
\midrule
\multirow{16}{*}{OD-Solver} & 1 & & 0.001 & 0.161 & 0.321 & 0.481 & 0.641 & 0.801 \\
& \multicolumn{2}{c}{(baseline)} & (na) & (na) & (na) & (na) & (na) & (na) \\
\cmidrule(lr){2-9}
& 2 & 0.5 & 0.988 & 0.942 & 0.897 & 0.851 & 0.805 & 0.759 \\
& & & (1.000) & (0.928) & (0.857) & (0.785) & (0.713) & (0.641) \\
\cmidrule(lr){2-9}
& 6 & 0.3 & 0.999$^*$ & 0.952 & 0.904 & 0.857 & 0.809 & 0.762 \\
& & & (1.000) & (0.914) & (0.828) & (0.741) & (0.656) & (0.569) \\[3pt]
& 6 & 0.5 & 0.998 & 0.960 & 0.923 & 0.885 & 0.847 & 0.809 \\
& & & (1.000) & (0.952) & (0.904) & (0.855) & (0.807) & (0.759) \\[3pt]
& 6 & 0.7 & 0.965 & 0.936 & 0.908 & 0.880 & 0.852 & 0.824 \\
& & & (1.000) & (0.965) & (0.931) & (0.896) & (0.862) & (0.827) \\
\cmidrule(lr){2-9}
& 10 & 0.3 & 0.999$^*$ & 0.953 & 0.906 & 0.859 & 0.813 & 0.766 \\
& & & (1.000) & (0.894) & (0.789) & (0.683) & (0.578) & (0.472) \\[3pt]
& 10 & 0.5 & 0.999$^*$ & {\bf 0.962}$^*$ & 0.924 & 0.887 & 0.850 & 0.813 \\
& & & (1.000) & (0.928) & (0.857) & (0.785) & (0.713) & (0.642) \\[3pt]
& 10 & 0.7 & 0.994 & {\bf 0.962}$^*$ & {\bf 0.930}$^*$ & {\bf 0.898}$^*$ & {\bf 0.866}$^*$ & {\bf 0.834}$^*$ \\
& & & (1.000) & (0.932) & (0.863) & (0.795) & (0.726) & (0.658) \\
\bottomrule
\end{tabular}
}
\caption{Right decision rates of EnSolver with different model and data settings. Bold numbers indicate the best values in each column. Asterisks ($^*$) indicate the best values among those with the same base model in each column. Numbers in parentheses are the theoretical rates obtained from Theorem~\ref{thrm:rdr}, with ``na'' indicating values that do not exist. The theoretical rates are computed using the hyper-parameter values in Table~\ref{tab:hyper}.}
\label{tab:ensolver}
\end{table}

\begin{table}[t]
\centering
\begin{tabular}{c ccc ccc}
\toprule
\multirow{2}{*}{Hyper-parameter} & \multicolumn{3}{c}{3E-Solver} & \multicolumn{3}{c}{OD-Solver} \\
\cmidrule(lr){2-4} \cmidrule(lr){5-7}
& $M=2$ & $M=6$ & $M=10$ & $M=2$ & $M=6$ & $M=10$ \\
\midrule
$N_{\mathcal{S}}$ & \multicolumn{6}{c}{$2.9\times 10^{12}$ (approx.)} \\
$\beta_{\text{min}}$ & 0.734 & 0.715 & 0.698 & 0.801 & 0.788 & 0.780 \\
$\beta_{\text{max}}$ & 0.739 & 0.748 & 0.748 & 0.819 & 0.819 & 0.821 \\
\bottomrule
\end{tabular}
\caption{Values of the hyper-parameters for computing the theoretical bounds in Table~\ref{tab:ensolver} and Table~\ref{tab:lensolver}.}
\label{tab:hyper}
\end{table}

In this experiment, we evaluate the right decision rates (defined in Definition~\ref{def:rdr}) of the EnSolver models. We measure the right decision rate by computing the proportion of right decisions a model makes on a test dataset containing both in-distribution and out-of-distribution samples. To give insights into the behavior of EnSolver, we vary the proportion of in-distribution data $\alpha \in \{ 0, 0.2, 0.4, 0.6, 0.8, 1 \}$ during testing. We compute both the empirical right decision rates as well as the theoretical rates obtained from Theorem~\ref{thrm:rdr} for comparison. The results of this experiment are reported in Table~\ref{tab:ensolver}.

From this table, we observe that EnSolver models are consistently better than the baselines when there are out-of-distribution data ($\alpha < 1$). Generally, $M=10$ and $\tau=0.7$ yield the best ensembles for both model types, with the OD-Solver ensembles achieving better rates than the 3E-Solver counterparts. In contrast to the baselines, the rates of the EnSolvers decay when $\alpha$ increases. This behavior is expected since the ensembles may skip some in-distribution data, which would be counted as incorrect decisions. We can also observe this effect when $\alpha$ is high and $\tau$ is decreased. In this case, the EnSolvers skip more in-distribution samples due to the low uncertainty threshold. When comparing the empirical rates with their theoretical values, the latter are generally good indicators of the former. This confirms the usefulness of Theorem~\ref{thrm:rdr} in predicting the actual right decision rates of EnSolvers.

\subsection{Success Rates of LEnSolvers}
\label{sec:lensolver_exp}

\begin{table}[!t]
\centering
\resizebox{0.89\textwidth}{!}{%
\begin{tabular}{ccc cccccc}
\toprule
\multicolumn{3}{c}{Model setting} & \multicolumn{6}{c}{Proportion of in-distribution data} \\
\cmidrule(lr){1-3} \cmidrule(lr){4-9}
Base type & M & $\tau$ & $\alpha=0$ & $\alpha=0.2$ & $\alpha=0.4$ & $\alpha=0.6$ & $\alpha=0.8$ & $\alpha=1$ \\
\midrule
\multirow{16}{*}{3E-Solver} & 1 & & 0.000$^*$ & 0.147 & 0.294 & 0.440 & 0.587 & 0.734 \\
& \multicolumn{2}{c}{(baseline)} & (na) & (na) & (na) & (na) & (na) & (na) \\
\cmidrule(lr){2-9}
& 2 & 0.5 & 0.000$^*$ & 0.427 & 0.671 & 0.797 & 0.850 & 0.871 \\
& & & (0.000) & (0.365) & (0.617) & (0.783) & (0.885) & (0.941) \\
\cmidrule(lr){2-9}
& 6 & 0.3 & 0.000$^*$ & 0.443 & 0.695 & 0.824 & 0.875$^*$ & 0.896 \\
& & & (0.000) & (0.278) & (0.463) & (0.578) & (0.644) & (0.677) \\[3pt]
& 6 & 0.5 & 0.000$^*$ & 0.459 & 0.706 & 0.821 & 0.861 & 0.875 \\
& & & (0.000) & (0.411) & (0.634) & (0.739) & (0.776) & (0.784) \\[3pt]
& 6 & 0.7 & 0.000$^*$ & 0.465 & 0.698 & 0.790 & 0.818 & 0.827 \\
& & & (0.000) & (0.481) & (0.709) & (0.795) & (0.815) & (0.817) \\
\cmidrule(lr){2-9}
& 10 & 0.3 & 0.000$^*$ & 0.438 & 0.691 & 0.820 & 0.875$^*$ & 0.897$^*$ \\
& & & (0.000) & (0.187) & (0.302) & (0.368) & (0.401) & (0.414) \\[3pt]
& 10 & 0.5 & 0.000$^*$ & 0.469 & 0.715$^*$ & 0.825$^*$ & 0.862 & 0.869 \\
& & & (0.000) & (0.325) & (0.487) & (0.553) & (0.572) & (0.574) \\[3pt]
& 10 & 0.7 & 0.000$^*$ & 0.477$^*$ & 0.715$^*$ & 0.811 & 0.838 & 0.839 \\
& & & (0.000) & (0.352) & (0.519) & (0.582) & (0.597) & (0.598) \\
\midrule
\multirow{16}{*}{OD-Solver} & 1 & & 0.001 & 0.161 & 0.321 & 0.481 & 0.641 & 0.801 \\
& \multicolumn{2}{c}{(baseline)} & (na) & (na) & (na) & (na) & (na) & (na) \\
\cmidrule(lr){2-9}
& 2 & 0.5 & 0.002 & 0.469 & 0.720 & 0.836 & 0.879 & 0.891 \\
& & & (0.000) & (0.418) & (0.680) & (0.830) & (0.906) & (0.937) \\
\cmidrule(lr){2-9}
& 6 & 0.3 & 0.001 & 0.480 & 0.739 & 0.857 & 0.898 & 0.911 \\
& & & (0.000) & (0.364) & (0.580) & (0.695) & (0.747) & (0.764) \\[3pt]
& 6 & 0.5 & 0.002 & 0.493 & 0.742 & 0.847 & 0.877 & 0.884 \\
& & & (0.000) & (0.458) & (0.689) & (0.784) & (0.812) & (0.815) \\[3pt]
& 6 & 0.7 & {\bf 0.004}$^*$ & 0.473 & 0.711 & 0.813 & 0.847 & 0.857 \\
& & & (0.000) & (0.488) & (0.720) & (0.806) & (0.826) & (0.828) \\
\cmidrule(lr){2-9}
& 10 & 0.3 & 0.002 & 0.484 & 0.744 & {\bf 0.861}$^*$ & {\bf 0.901}$^*$ & {\bf 0.913}$^*$ \\
& & & (0.000) & (0.295) & (0.460) & (0.540) & (0.571) & (0.579) \\[3pt]
& 10 & 0.5 & 0.003 & {\bf 0.497}$^*$ & {\bf 0.747}$^*$ & 0.851 & 0.881 & 0.890 \\
& & & (0.000) & (0.381) & (0.564) & (0.634) & (0.651) & (0.652) \\[3pt]
& 10 & 0.7 & {\bf 0.004}$^*$ & 0.495 & 0.736 & 0.830 & 0.855 & 0.860 \\
& & & (0.000) & (0.388) & (0.572) & (0.641) & (0.657) & (0.658) \\
\bottomrule
\end{tabular}
}
\caption{Success rates of LEnSolver with $T=3$ for different model and data settings. Bold numbers indicate the best values in each column. Asterisks ($^*$) indicate the best values among those with the same base model in each column. Numbers in parentheses are the theoretical rates obtained from Theorem~\ref{thrm:success}, with ``na'' indicating values that do not exist. The theoretical rates are computed using the hyper-parameter values in Table~\ref{tab:hyper}.}
\label{tab:lensolver}
\end{table}

We now focus on the LEnSolver models and evaluate their success rates (defined in Definition~\ref{def:success_rate}) by simulating 400,000 actual CAPTCHA cracking scenarios, where the CAPTCHAs are sampled from both in-distribution and out-of-distribution test sets with different $\alpha$ values. In this experiment, we consider LEnSolvers with the maximum number of skips $T=3$. We report both the empirical success rates and the theoretical rates obtained from Theorem~\ref{thrm:success} in Table~\ref{tab:lensolver}. We note that the baselines do not have the skip mechanism and will make a prediction immediately on any given CAPTCHA. Thus, their success rates are exactly the same as their right decision rates in Section~\ref{sec:ensolver_exp}.

From Table~\ref{tab:lensolver}, we can observe that the LEnSolvers are consistently better than the baselines in terms of success rates. The OD-Solver ensembles are generally better than the 3E-Solver counterparts, with ensembles of size $M=10$ achieving the best rates among ensembles with the same base model type. When comparing ensembles of the same size, the optimal value of $\tau$ depends on the ensemble size and the value of $\alpha$, with $\tau = 0.5$ giving a reasonable performance in all cases. We also observe that the success rates of LEnSolvers increase with $\alpha$. This behavior is expected since lower values of $\alpha$ mean a model will likely encounter out-of-distribution CAPTCHAs and it will be forced to make a prediction after reaching the maximum number of skips. However, even with $\alpha$ as low as $0.4$, LEnSolvers can still achieve decent success rates relatively to $\alpha = 1$. For the LEnSolvers, the theoretical rates are also reasonable indicators of the empirical rates, which confirms the usefulness of our Theorem~\ref{thrm:success}.

\subsection{Effects of Maximum Number of Skips}
\label{sec:maxskip_exp}

This experiment investigates the effects of the maximum number of skips $T$ on the success rates of LEnSolvers. For this purpose, we fix $M=10$, $\tau=0.5$ and plot the empirical and theoretical success rates of LEnSolver with $T \in \{ 0, 1, \ldots, 5 \}$. Figure~\ref{fig:lensoler_rate} shows the plots for the two base model types and $\alpha \in \{ 0.2, 0.5, 0.8 \}$. From the figure, we can observe that the empirical success rates increase with $T$, and our theoretical rates correctly reflect this trend. Furthermore, the theoretical lines also correctly capture the dynamics of the corresponding empirical lines (e.g., both lines plateau when $T \ge 2$ and $\alpha=0.8$), although the theoretical lines are not very tight to the empirical lines, especially for larger $T$. This is because we derive our bounds using $\beta_{\text{min}}$ and $\beta_{\text{max}}$ instead of all $\beta_i$'s, leading to relatively loose bounds for the skip rate $p \left( m(x) = s_{\text{skip}} \right)$ in Lemma~\ref{lem:skip_rate} and the final success rate in Theorem~\ref{thrm:success}. Nevertheless, the theoretical rates are still reasonable lower bounds of the actual success rates, which are much higher in practice.

\begin{figure}[tb]
    \centering
    \begin{tabular}{ccc}
        \subfloat[$\alpha=0.2$]{\includegraphics[width=0.31\textwidth]{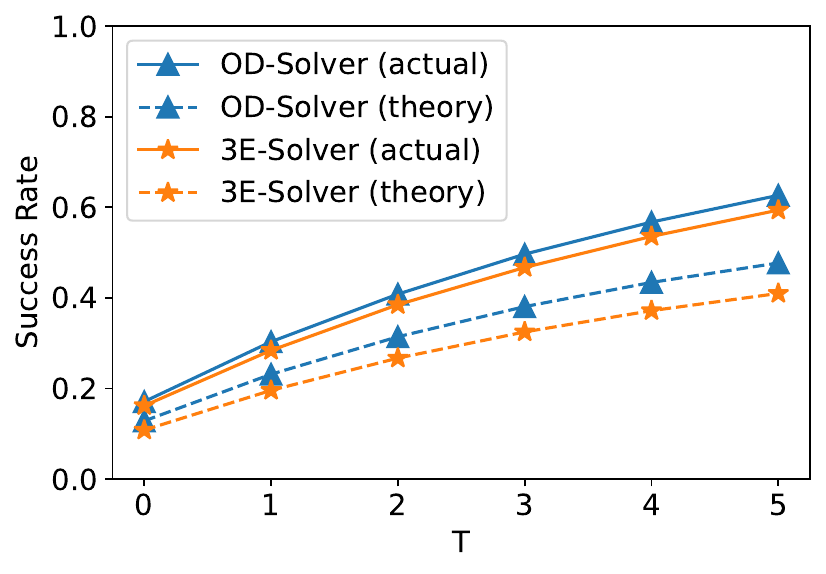}}
        \hspace{1mm}
        \subfloat[$\alpha=0.5$]{\includegraphics[width=0.31\textwidth]{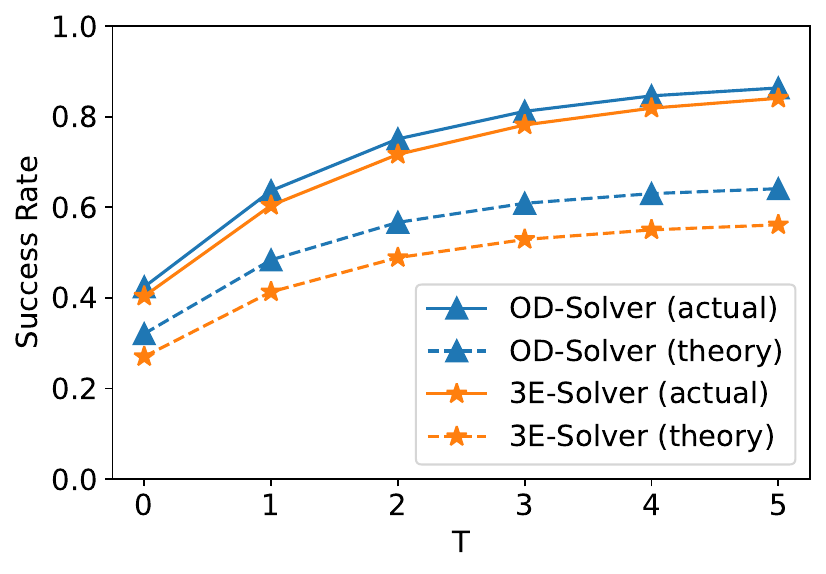}}
        \hspace{1mm}
        \subfloat[$\alpha=0.8$]{\includegraphics[width=0.31\textwidth]{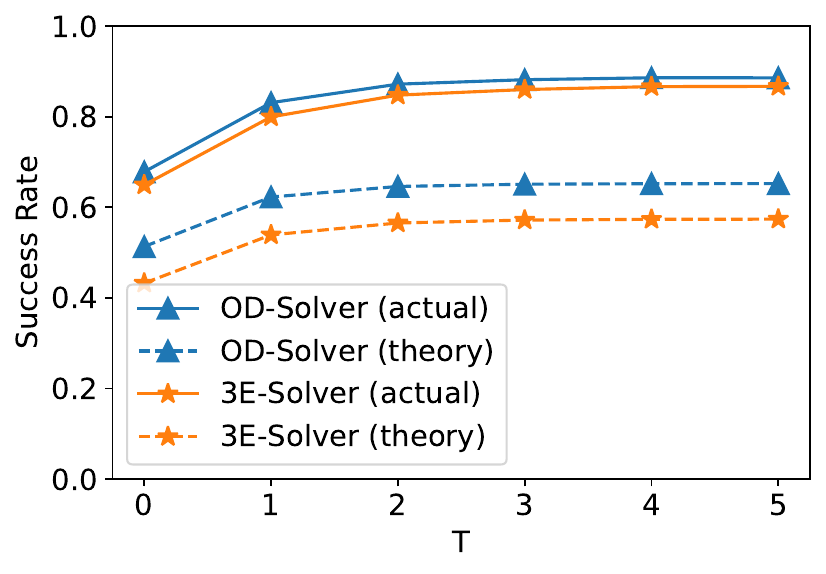}}
    \end{tabular}
    \caption{Actual and theoretical success rates of LEnSolver with respect to the maximum number of skips (T) for different values of $\alpha$ and base model types.}
    \label{fig:lensoler_rate}
\end{figure}

\subsection{Discussions}
\label{sec:discuss}

In our experiment results in Tables~\ref{tab:ensolver} and~\ref{tab:lensolver}, there are some cases where the theoretical rates are higher than the empirical rates (e.g., when $M=6$, $\tau=0.7$ and $\alpha=0.8$ in Table~\ref{tab:ensolver} or when $M=2$ and $\alpha=1$ in Table~\ref{tab:lensolver}), although our theorems prove lower bound results for these theoretical rates. This is due to the uniform assumptions (A2) and (A3) in Section~\ref{sec:theory-setting} being violated in practice. Thus, a potential direction for future work is to relax these assumptions to develop an improved theory for our methods. Nevertheless, even in these cases, our experiment results still show that the theoretical rates are close to the empirical rates. 

Another issue with our theoretical rates is the relatively loose bound compared to the empirical rates for large $M$. This is due to the fact that we derive our bounds using $\beta_{\text{min}}$ and $\beta_{\text{max}}$. In principle, we can make the bounds tighter by using all the $\beta_i$'s of the base models. But this would make the bounds more complicated to derive and compute for large ensembles. A potential future work would be to derive tighter bounds that are both simple and easy to compute. It is also worth to emphasize that even though our bounds are loose, they are non-trivial and can be used as pessimistic indicators of the solvers' effectiveness, and the actual success rates would be much higher.

\section{Conclusion}

We proposed EnSolver and LEnSolver, novel end-to-end uncertainty-aware CAPTCHA solvers that can detect and skip out-of-distribution inputs. Our solvers use a deep ensemble of base CAPTCHA solvers to obtain uncertainty estimates and use them to make decisions. We prove theoretical guarantees on the effectiveness of the approaches and show empirically that our solvers can achieve good success rates in the presence of out-of-distribution data. Our work is potentially helpful for security experts to better understand the capability of automatic CAPTCHA solvers and improve the defense against these attacks.

\newpage


\acks{This work was partially supported by the U.S.~National Science Foundation (Award: 2331908), US National Security Agency (Award: H982302110324) and Microsoft Security AI. The views expressed are those of the authors
only and not of the funding agencies.}

\appendix

\section{Proofs of Theoretical Results}
\label{sec:proofs}

\minisection{Additional notations.}
Recall that the EnSolver $m$ consists of $M$ base models $m_1, m_2, \ldots, \allowbreak m_M$. For any $x \in \mathcal{X}$ and $s \in \mathcal{S}$, let $n(x, s)$ be the number of base models predicting the output $s$ for $x$. That is,
\begin{align*}
n(x, s) = \left| \{ i : m_i(x) = s, ~ i \in \{ 1, 2, \ldots, M \} \} \right|.
\end{align*}
For any $i \in \{ 1, 2, \ldots, M \}$, let $\mathcal{K}_i$ be the set of all subsets of size $i$ of $\{ 1, 2, \ldots, M \}$: 
\begin{align*}
\mathcal{K}_i = \{ \kappa \subseteq \{ 1, 2, \ldots, M \} : |\kappa| = i \}.
\end{align*}
For any $\kappa \in \mathcal{K}_i$, we denote $\widebar{\kappa} = \{ 1, 2, \ldots, M \} \setminus \kappa$. We also let $\widebar{\mathcal{K}}_i$ be the set of all $\widebar{\kappa}$'s:
\begin{align*}
\widebar{\mathcal{K}}_i = \{ \widebar{\kappa} = \{ 1, 2, \ldots, M \} \setminus \kappa : \kappa \in \mathcal{K}_i \}.
\end{align*}
We now prove the following lemma that will be used in our proofs.

\begin{lem}
    The following inequality holds:
    \begin{align*}
        \sum_{i=M(1-\tau)+1}^{M} \begin{pmatrix}M \\ i \end{pmatrix}\frac{1}{(N_{\mathcal{S}})^i} < \frac{\mathcal{E}(M, N_{\mathcal{S}}, \tau)}{N_{\mathcal{S}}-1}.
    \end{align*}
\label{lem1}
\end{lem}

\begin{proof}
    For brevity, we denote $k=M(1-\tau)+1$.
    For any positive integer $M$ and any integer $i$ with $0\le i \le M$, we have:
    \begin{align*}
        \begin{pmatrix}M \\ i \end{pmatrix} \le \begin{pmatrix}M \\ \lfloor M/2 \rfloor \end{pmatrix}.
    \end{align*}
    
    Therefore,
    \begin{align*}
        \sum_{i=k}^{M} \begin{pmatrix}M \\ i \end{pmatrix}\frac{1}{(N_{\mathcal{S}})^i}
        &\le \sum_{i=k}^{M} \begin{pmatrix}M \\ \lfloor M/2 \rfloor \end{pmatrix}\frac{1}{(N_{\mathcal{S}})^i} \\
        &= \begin{pmatrix}M \\ \lfloor M/2 \rfloor \end{pmatrix}\frac{1}{(N_{\mathcal{S}})^k}\sum_{i=0}^{M-k}\frac{1}{(N_{\mathcal{S}})^i} \\
        &= \begin{pmatrix}M \\ \lfloor M/2 \rfloor \end{pmatrix}\frac{1}{(N_{\mathcal{S}})^k}.\frac{1-\frac{1}{(N_{\mathcal{S}})^{M+1-k}}}{1-\frac{1}{N_{\mathcal{S}}}} \\
        &< \begin{pmatrix}M \\ \lfloor M/2 \rfloor \end{pmatrix}\frac{1}{(N_{\mathcal{S}})^k}.\frac{1}{1-\frac{1}{N_{\mathcal{S}}}} \\
        &= \begin{pmatrix}M \\ \lfloor M/2 \rfloor \end{pmatrix} \frac{1}{(N_{\mathcal{S}})^{k-1}(N_{\mathcal{S}} - 1)} = \frac{\mathcal{E}(M, N_{\mathcal{S}}, \tau)}{N_{\mathcal{S}}-1}.
    \end{align*}
    {\vskip -0.7cm}
\end{proof}

\subsection{Proof of Lemma~\ref{lem:oeb_out}}
\label{proof:oeb_out}

Let $k=M(1-\tau)+1$. We have:
\begin{align*}
    p &\left( m(x) \notin \{ s_x, s_{\text{skip}} \} \,|\, x \in \mathcal{X}^{\text{out}} \right) \\
    &= p \left( \exists s \in \mathcal{S} \setminus \{ s_x \}, ~ m(x) = s \,|\, x \in \mathcal{X}^{\text{out}} \right) \\
    &\le p \left( \exists s \in \mathcal{S} \setminus \{ s_x \}, ~ n(x,s) \ge k \,|\, x \in \mathcal{X}^{\text{out}} \right) \\
    &\le \sum_{s\in \mathcal{S} \setminus \{s_x\}} p \left( n(x,s) \ge k \,|\, x \in \mathcal{X}^{\text{out}} \right) \\
    &= \sum_{s\in \mathcal{S} \setminus \{s_x\}} \, \sum_{i=k}^{M} \begin{pmatrix}M \\ i \end{pmatrix}\frac{1}{(N_{\mathcal{S}})^i}\left(1-\frac{1}{N_{\mathcal{S}}}\right)^{M-i} \tag{assumption A3} \\
    &< \sum_{s\in \mathcal{S} \setminus \{s_x\}} \, \sum_{i=k}^{M} \begin{pmatrix}M \\ i \end{pmatrix}\frac{1}{(N_{\mathcal{S}})^i} = (N_{\mathcal{S}}-1)\sum_{i=k}^{M} \begin{pmatrix}M \\ i \end{pmatrix}\frac{1}{(N_{\mathcal{S}})^i} < \mathcal{E}(M, N_{\mathcal{S}}, \tau),
\end{align*}
where the last inequality comes from Lemma~\ref{lem1}.
$\hfill \blacksquare$

\subsection{Proof of Lemma~\ref{lem:oeb_in}}
\label{proof:oeb_in}

Let $k=M(1-\tau)+1$. For any $s \in \mathcal{S}$ such that $s\ne s_x$, we have:
\begin{align*}
    p &\left( n(x, s) \ge k \,|\, x \in \mathcal{X}^{\text{in}} \right) \\
    &= \sum_{i=k}^{M} \, \sum_{\kappa \in \mathcal{K}_i} \left( \, \prod_{j\in \kappa}\left(\frac{1-\beta_j}{N_{\mathcal{S}}-1}\right)\prod_{t\in \widebar\kappa}\left(1-\frac{1-\beta_t}{N_{\mathcal{S}}-1}\right) \right) \tag{assumption A2} \\
    &\le \sum_{i=k}^{M} \, \sum_{\kappa\in \mathcal{K}_i}\left(\frac{1-\beta_{\text{min}}}{N_\mathcal{S}-1}\right)^i\left(1-\frac{1-\beta_{\text{max}}}{N_{\mathcal{S}}-1}\right)^{M-i} \\
    &= \sum_{i=k}^{M}\begin{pmatrix}M \\ i \end{pmatrix}\left(\frac{1-\beta_{\text{min}}}{N_\mathcal{S}-1}\right)^i\left(1-\frac{1-\beta_{\text{max}}}{N_{\mathcal{S}}-1}\right)^{M-i} \\
    &< \sum_{i=k}^{M} \begin{pmatrix}M \\ i \end{pmatrix}\frac{1}{(N_{\mathcal{S}})^i} \tag{assumption A1 and $1-\frac{1-\beta_{\text{max}}}{N_{\mathcal{S}}-1} < 1$} \\
    &< \frac{\mathcal{E}(M, N_{\mathcal{S}}, \tau)}{N_{\mathcal{S}}-1}. \tag{lemma~\ref{lem1}}
\end{align*}
Therefore,
\begin{align*}
    p \left( \exists s \ne s_x, ~ n(x, s) \ge k \,|\, x \in \mathcal{X}^{\text{in}} \right)
    &\le \sum_{s\in \mathcal{S} \setminus \{s_x\}} p \left( n(x, s) \ge k \,|\, x \in \mathcal{X}^{\text{in}} \right) \\
    &< \sum_{s\in \mathcal{S} \setminus \{s_x\}} \frac{\mathcal{E}(M, N_{\mathcal{S}}, \tau)}{N_{\mathcal{S}}-1} = \mathcal{E}(M, N_{\mathcal{S}}, \tau).
\end{align*}
$\hfill \blacksquare$

\subsection{Proof of Theorem~\ref{thrm:rdr}}
\label{proof:rdr}

From Definition~\ref{def:rdr}, we have:
\begin{align}
R_{\text{rd}}(m) &=
p(m(x) = s_x, x \in \mathcal{X}^{\text{in}}) + p(m(x) \in \{s_x, s_{\text{skip}}\}, x \in \mathcal{X}^{\text{out}}) \notag \\
&= p(m(x) = s_x \,|\, x \in \mathcal{X}^{\text{in}}) \, p(x \in \mathcal{X}^{\text{in}}) + p(m(x) \in \{s_x, s_{\text{skip}}\} \,|\, x \in \mathcal{X}^{\text{out}}) \, p(x \in \mathcal{X}^{\text{out}}) \notag \\
&= \alpha \, p(m(x) = s_x \,|\, x \in \mathcal{X}^{\text{in}}) + (1 - \alpha) \, p(m(x) \in \{s_x, s_{\text{skip}}\} \,|\, x \in \mathcal{X}^{\text{out}}). \label{eq:rm}
\end{align}
From Lemma~\ref{lem:oeb_out},
\begin{align}
p(m(x) \in \{s_x, s_{\text{skip}}\} \,|\, x \in \mathcal{X}^{\text{out}})
&= 1 - p(m(x) \notin \{s_x, s_{\text{skip}}\} \,|\, x \in \mathcal{X}^{\text{out}}) \notag \\
&\ge 1 - \mathcal{E}(M, N_{\mathcal{S}}, \tau).
\label{eq:thrm_rate_1}
\end{align}
Thus, we only need to lower bound $p(m(x) = s_x \,|\, x \in \mathcal{X}^{\text{in}})$. Let $k=M(1-\tau)+1$ and consider the following two events:
\begin{align*}
(A) &: n(x, s_x) \ge k, \\
(B) &: n(x, s) < n(x, s_x), \forall s \ne s_x.
\end{align*}
Note that:
\begin{align}
p(m(x) = s_x \,|\, x \in \mathcal{X}^{\text{in}}) 
&= p(A \wedge B \,|\, x \in \mathcal{X}^{\text{in}}) \notag \\
&= p(A \,|\, x \in \mathcal{X}^{\text{in}}) - p(A \wedge \neg B \,|\, x \in \mathcal{X}^{\text{in}}). \label{eq:thrm_rate_2}
\end{align}
We have:
\begin{align}
p(A \,|\, x \in \mathcal{X}^{\text{in}}) &= \sum_{i=k}^M \; \sum_{\kappa \in \mathcal{K}_i} \left( \prod_{j \in \kappa} \beta_j \right) \left( \prod_{t \in \widebar\kappa} ( 1-\beta_t ) \right) \notag \\
&\ge \sum_{i=k}^M \; \sum_{\kappa \in \mathcal{K}_i} \beta_{\text{min}}^i \, ( 1-\beta_{\text{max}} )^{M-i} \notag \\
&= \sum_{i=k}^M \begin{pmatrix} M \\ i \end{pmatrix} \beta_{\text{min}}^i \, ( 1-\beta_{\text{max}} )^{M-i}. \label{eq:thrm_rate_3}
\end{align}
We also have:
\begin{align}
p(A \wedge \neg B \,|\, x \in \mathcal{X}^{\text{in}}) &= p \left( \left\{ \begin{array}{l} n(x, s_x) \ge k \\ \exists s \ne s_x, n(x, s) \ge n(x, s_x) \end{array} \right. \Big| ~ x \in \mathcal{X}^{\text{in}} \right) \notag \\
&\le p \left( \left\{ \begin{array}{l} n(x, s_x) \ge k \\ \exists s \ne s_x, n(x, s) \ge k \end{array} \right. \Big| ~ x \in \mathcal{X}^{\text{in}} \right) \notag \\
&\le p \left( \exists s \ne s_x, n(x, s) \ge k ~|~ x \in \mathcal{X}^{\text{in}} \right) \notag \\
&< \mathcal{E}(M, N_{\mathcal{S}}, \tau), \label{eq:thrm_rate_4}
\end{align}
where the last inequality is from Lemma~\ref{lem:oeb_in}. Combining~\eqref{eq:thrm_rate_2},~\eqref{eq:thrm_rate_3} and~\eqref{eq:thrm_rate_4}, we have:
\begin{align}
p(m(x) = s_x \,|\, x \in \mathcal{X}^{\text{in}}) 
&> \sum_{i=k}^M \begin{pmatrix} M \\ i \end{pmatrix} \beta_{\text{min}}^i \, ( 1-\beta_{\text{max}} )^{M-i} - \mathcal{E}(M, N_{\mathcal{S}}, \tau). \label{eq:thrm_rate_5}
\end{align}
From~\eqref{eq:rm},~\eqref{eq:thrm_rate_1} and~\eqref{eq:thrm_rate_5}, we have:
\begin{align*}
R_{\text{rd}}(m)
&> \alpha \left( \, \sum_{i=k}^M \begin{pmatrix} M \\ i \end{pmatrix} \beta_{\text{min}}^i \, ( 1-\beta_{\text{max}} )^{M-i} - \mathcal{E}(M, N_{\mathcal{S}}, \tau) \right) + (1 - \alpha) \left( 1 - \mathcal{E}(M, N_{\mathcal{S}}, \tau) \right) \\
&= \alpha \sum_{i=k}^M \begin{pmatrix} M \\ i \end{pmatrix} \beta_{\text{min}}^i \, ( 1-\beta_{\text{max}} )^{M-i} + (1 - \alpha) - \mathcal{E}(M, N_{\mathcal{S}}, \tau).
\end{align*}
$\hfill \blacksquare$

\subsection{Proof of Lemma~\ref{lem:correct_rate}}
\label{proof:correct_rate}

Note that:
\begin{align*}
    p \left( m(x) = s_x \right) = \alpha \, p(m(x) = s_x \,|\, x \in \mathcal{X}^{\text{in}}) + (1-\alpha) \, p(m(x) = s_x \,|\, x \in \mathcal{X}^{\text{out}}).
\end{align*}
From~\eqref{eq:thrm_rate_5} in the proof of Theorem~\ref{thrm:rdr} above, we have:
\begin{align*}
    p(m(x) = s_x \,|\, x \in \mathcal{X}^{\text{in}}) > \sum_{i=k}^M \begin{pmatrix} M \\ i \end{pmatrix} \beta_{\text{min}}^i \, ( 1-\beta_{\text{max}} )^{M-i} - \mathcal{E}(M, N_{\mathcal{S}}, \tau).
\end{align*}
In addition, $p(m(x) = s_x \,|\, x \in \mathcal{X}^{\text{out}}) > 0$. Therefore,
\begin{align*}
    p \left( m(x) = s_x \right) > \alpha \sum_{i=k}^{M}\begin{pmatrix} M \\ i \end{pmatrix}\beta_{\text{min}}^i (1-\beta_{\text{max}})^{M-i} - \alpha \, \mathcal{E}(M, N_{\mathcal{S}}, \tau).
\end{align*}
$\hfill \blacksquare$

\subsection{Proof of Lemma~\ref{lem:skip_rate}}
\label{proof:skip_rate}

Similar to the proof above, we decompose the probability $p \left( m(x) = s_{\text{skip}} \right)$ as:
    \begin{align}
        p \left( m(x) = s_{\text{skip}} \right) = \alpha \, p(m(x) = s_{\text{skip}} \,|\, x \in \mathcal{X}^{\text{in}}) + (1-\alpha) \, p(m(x) = s_{\text{skip}} \,|\, x \in \mathcal{X}^{\text{out}}). \label{eq:lem_skip_1}
    \end{align}
We first compute a lower bound for the probability on $\mathcal{X}^{\text{out}}$. Let $k=M(1-\tau)+1$. Since $m(x)=s_{\text{skip}}$ if and only if $n(x, s) < k ~ \forall s\in \mathcal{S}$, we have:
\begin{align*}
    p(m(x) = s_{\text{skip}} \,|\, x \in \mathcal{X}^{\text{out}})
    &= 1 - p \left( \exists s \in \mathcal{S}, ~ n(x, s) \ge k \,|\, x \in \mathcal{X}^{\text{out}} \right) \\
    &\ge 1 - \sum_{s \in \mathcal{S}} p \left( n(x, s) \ge k \,|\, x \in \mathcal{X}^{\text{out}} \right).
\end{align*}
Using the same argument in the proof of Lemma~\ref{lem:oeb_out}, for each $s\in \mathcal{S}$, we have:
\begin{align*}
    p \left( n(x, s) \ge k \,|\, x \in \mathcal{X}^{\text{out}} \right) < \frac{\mathcal{E}(M, N_{\mathcal{S}}, \tau)}{N_{\mathcal{S}}-1}.
\end{align*}
This implies:
\begin{align}
    p(m(x) = s_{\text{skip}} \,|\, x \in \mathcal{X}^{\text{out}}) > 1 - \sum_{s \in \mathcal{S}}\frac{\mathcal{E}(M, N_{\mathcal{S}}, \tau)}{N_{\mathcal{S}}-1} = 1 - \frac{N_{\mathcal{S}}}{N_{\mathcal{S}}-1}\mathcal{E}(M, N_{\mathcal{S}}, \tau). \label{eq:lem_skip_2}
\end{align}
Next, we compute a lower bound for the probability on $\mathcal{X}^{\text{in}}$. We have:
\begin{align*}
    p(m(x) &= s_{\text{skip}} \,|\, x \in \mathcal{X}^{\text{in}}) \\
    &= p(\forall s\in \mathcal{S}, n(x, s) < k \,|\, x \in \mathcal{X}^{\text{in}}) \\
    &= p((n(x, s_x) < k) \wedge (\forall s\ne s_x, n(x, s) < k) \,|\, x \in \mathcal{X}^{\text{in}}) \\
    &= p((n(x, s_x) < k) \wedge \neg (\exists s\ne s_x, n(x, s) \ge k \,|\, x \in \mathcal{X}^{\text{in}}) \\
    &\ge p(n(x, s_x) < k \,|\, x \in \mathcal{X}^{\text{in}}) - p(\exists s\ne s_x, n(x, s) \ge k \,|\, x \in \mathcal{X}^{\text{in}}),
\end{align*}
where we use the fact that $p(A \wedge \neg B) \ge p(A) - p(B)$. From Lemma~\ref{lem:oeb_in}, we know that $p(\exists s\ne s_x, n(x, s) \ge k \,|\, x \in \mathcal{X}^{\text{in}}) < \mathcal{E}(M, N_{\mathcal{S}}, \tau)$. Besides, using a similar argument when deriving~\eqref{eq:thrm_rate_3} in the proof of Theorem~\ref{thrm:rdr}, we have:
\begin{align*}
    p(n(x, s_x) < k \,|\, x \in \mathcal{X}^{\text{in}})
    &= \sum_{i=0}^{k-1} \, \sum_{\kappa \in \mathcal{K}_i} \left( \prod_{j\in \kappa}\beta_j \right) \left( \prod_{t\in \widebar\kappa}(1-\beta_t) \right) \\
    &\ge \sum_{i=0}^{k-1} \, \sum_{\kappa \in \mathcal{K}_i}\beta_{\text{min}}^i (1-\beta_{\text{max}})^{M-i} \\
    &= \sum_{i=0}^{k-1} \begin{pmatrix} M \\ i \end{pmatrix}\beta_{\text{min}}^i (1-\beta_{\text{max}})^{M-i}.
\end{align*}
Thus,
\begin{align}
    p(m(x) = s_{\text{skip}} \,|\, x \in \mathcal{X}^{\text{in}}) > \sum_{i=0}^{k-1} \begin{pmatrix} M \\ i \end{pmatrix}\beta_{\text{min}}^i (1-\beta_{\text{max}})^{M-i} - \mathcal{E}(M, N_{\mathcal{S}}, \tau). \label{eq:lem_skip_3}
\end{align}
From \eqref{eq:lem_skip_1}, \eqref{eq:lem_skip_2} and \eqref{eq:lem_skip_3}, we have:
\begin{align*}
    p \left( m(x) = s_{\text{skip}} \right)
    &> \alpha\left( \, \sum_{i=0}^{k-1}\begin{pmatrix} M \\ i \end{pmatrix}\beta_{\text{min}}^i (1-\beta_{\text{max}})^{M-i} - \mathcal{E}(M, N_{\mathcal{S}}, \tau) \right) \\
    &\qquad + (1-\alpha)\left( 1 - \frac{N_{\mathcal{S}}}{N_{\mathcal{S}}-1}\mathcal{E}(M, N_{\mathcal{S}}, \tau) \right) \\
    &= \alpha \sum_{i=0}^{k-1}\begin{pmatrix} M \\ i \end{pmatrix}\beta_{\text{min}}^i (1-\beta_{\text{max}})^{M-i} + (1-\alpha) - \frac{N_{\mathcal{S}}-\alpha}{N_{\mathcal{S}}-1} \, \mathcal{E}(M, N_{\mathcal{S}}, \tau).
\end{align*}
$\hfill \blacksquare$

\vskip 0.2in
\bibliography{jmlr_main}

\end{document}